\documentclass[11pt]{article}

\usepackage[
  shownumpages,          % Показать номера страниц
  nodefaultlogos,        % Убрать зашитые логотипы МФТИ / ИСП РАН / Иннополиса / BRAIn
  citingstyle=numbers,     % Стиль цитирования: numbers | authoryear | super
  bibliostyle=unsrtnat,    % Стиль библиографии: unsrtnat | plainnat | abbrvnat
  bibfile=references,      % Имя .bib файла (без расширения)
]{brainlab}

\usepackage{shortcuts}
\usepackage{wrapfig}

\usepackage{nicefrac}
\usepackage{booktabs}
\usepackage{multirow}
\usepackage{graphicx}
\usepackage{algorithm}

\setbrainmeta{
  title={Bridging Compressed Communication and Local Steps via Spectral Method in Federated Learning},
  authors={
    Dmitriy Bystrov\textsuperscript{1}, 
    George Chirkov\textsuperscript{1},
    Daniil Medyakov\textsuperscript{1,2}, 
    Dmitry Bylinkin\textsuperscript{1,2}, 
    Aleksandr Beznosikov\textsuperscript{1,2,3}
  },
  affiliations={
    \textsuperscript{1}Basic Research of Artificial Intelligence Laboratory (BRAIn Lab) \\
    \textsuperscript{2}Federated Learning Problems Laboratory \\
    \textsuperscript{3}Innopolis University \\
  },
  abstract={
    Federated Learning (FL) enables collaborative training of complex models across large-scale networks of edge devices while keeping data localized. However, data heterogeneity and high communication overhead remain fundamental bottlenecks. Recently, the spectral optimizer Muon has demonstrated remarkable advantages over Adam-like methods in centralized training. While recent efforts have attempted to transition Muon to FL, naive integrations into traditional architectures fail short. Crucially, existing approaches do not support the the communication compression and local steps --- the two foundational techniques of efficient FL. To address this limitation, we introduce a novel spectral FL method that utilizes both local updates and compression. We establish rigorous convergence guarantees under mild assumptions, notably eliminating the restrictive bounded heterogeneity condition. Empirical evaluation on standard image classification benchmarks validates the efficiency and robustness of our approach.
  },
}

\begin{document}
\begin{mainpart}

\section{Introduction}\label{sec:intro}
Federated Learning (FL) \citep{mcmahan2017communication} is a paradigm of training a shared model across a set of edge devices (often called workers or clients). The clients collect data locally and use it to perform updates, which are aggregated by the server.
Formally, the problem statement is as follows:
\begin{align*}
    \min_{x\in \mathbb{R}^{d_1\times d_2}} \left[f(x) = \frac1M \sum_{i=1}^M f_i(x) \right],
\end{align*}
where $x$ denotes the model weights, $f_i$ is the empirical risk \citep{shalev2014understanding} of client $i$ computed on its local dataset, and $M$ is the total number of clients.

The communication bottleneck remains a central challenge in federated learning \citep{li2020federated}. The source of this issue is two-fold.
Firstly, due to the number of devices in the network, coordinating them becomes a substantial asynchronous workload \citep{le2024exploring}. Thus, reducing communication frequency is important in practice. Algorithmically, the problem can be addressed by local steps. Here, each client performs several local update steps between aggregations \citep{gorbunov2021local}.
Secondly, an individual model may contain up to billions of parameters \citep{zhou2025comprehensive}, and its full transmission severely slows down the aggregation process. Communication compression aims to reduce the size of each message by removing redundant or low-importance information. In practice, model size can be reduced by orders of magnitude with essentially no performance loss \citep{haddadpour2021federated,li2023analysis}.

Modern machine learning problems are growing in complexity and demand advanced optimization techniques. Currently, the most widely used solutions rely on adaptive algorithms like AdaGrad \citep{duchi2011adaptive}, Adam \citep{kingma2014adam}, AdamW \citep{loshchilov2017decoupled}, and Shampoo \citep{gupta2018shampoo}. These methods adapt the update to the local geometry of the problem using current optimization statistics. The paradigm has become the gold standard since its development \citep{liu2025recent}, and it was successfully applied to the FL setting \citep{deng2020adaptive,reddi2020adaptive,wu2023faster}.

A different strategy emerged recently: an adaptation to the spectrum of matrix updates instead of the geometry of the parameter space \citep{jordan2024muon,riabinin2025gluon,pethick2025training}. Subsequent work has empirically demonstrated its advantages over classic optimizers \citep{liu2025muon,semenov2025benchmarking}, and \texttt{Muon} \citep{jordan2024muon} has attracted significant attention as the most efficient. Notably, \texttt{Muon} can be viewed as a Linear Minimization Oracle (LMO) under the spectral matrix norm, and this perspective yields theoretical improvements over classic approaches \citep{bernstein2024old,pethick2025training}. 
Hence, transferring this paradigm to FL can yield methods that surpass adaptive approaches.
%valuable, given that both issues of communication bottleneck are properly adressed.
%A plenty of works followed, developing various \texttt{Muon}-like methods \citep{ahn2025dion,liu2026muon,zhang2026teon}, including \texttt{FedMuon} in FL \citep{takezawa2025fedmuon}. 

\section{Related Works}\label{sec:related_works}

\subsection{Methods for Communication Compression} 
Compression aims to reduce the size of the transmitted message while retaining nearly all essential information. 
Two common approaches are quantization ~--~ reducing the bit precision of each parameter \citep{alistarh2017qsgd,bernstein2018signsgd}, and sparsification ~--~ removing a subset of the vector's values \citep{lin2017deep,stich2018sparsified}.

Theoretically, compressors fall into two commonly studied classes, unbiased and contractive \citep{beznosikov2023biased}. For unbiased compressors, the expectation of the output is equal to the original vector, while contractive compressors do not have this property. 
Unbiased compression is conceptually similar to the stochasticity of gradients in the finite-sum setting. 
In both cases, the distortion is bounded but does not vanish as the iterate approaches the optimum. Hence, techniques developed for the stochastic setting carry over. Namely, the majority of works in this line are driven by Variance Reduction (VR) \citep{johnson2013accelerating,schmidt2017minimizing} mechanisms that were initially designed for stochastic gradients \citep{gorbunov2021marina,beznosikov2022distributed,tyurin2022dasha,mishchenko2025distributed}. However, unbiased compressors tend to perform poorly in practice \citep[Figure 5]{beznosikov2023biased}.
Thus, we focus on contractive compression. 

Notably, simple \texttt{SGD} with an arbitrary contractive compressor does not converge. To overcome the obstacle, special techniques are required, the most widely used of which being Error Feedback (\texttt{EF}). First appeared as a heuristic \citep{seide20141}, \texttt{EF} later received a theoretical foundation with arbitrary compression \citep{karimireddy2019error,stich2020error,beznosikov2023biased}.
The core concept of \texttt{EF} is combining the current message with the error from the previous compression, which compensates for inaccuracies at previous steps. Crucially, \texttt{EF} allows one to establish convergence. \texttt{EF21} is a noteworthy advancement that improves convergence rates in the non-convex setup by utilizing a distinct error accounting procedure
\citep{richtarik2021ef21}. The idea is to compress the residuals instead of the full message. Both \texttt{EF} and \texttt{EF21} have found numerous applications in FL \citep{karimireddy2019error,li2023analysis,bylinkin2025robust}.

\subsection{Local Steps} 
Another approach to tackling the communication bottleneck is local steps. Instead of calculating a single gradient, each client performs several steps over its local dataset and communicates the accumulated model update. This way, the costly server-side aggregation is performed less often. The technique is extensively studied \citep{woodworth2020local,khaled2020tighter,koloskova2020unified,gorbunov2021local}.

Local steps have a crucial downside: clients lose the alignment of their training trajectories between aggregations.
%Thus, the technique tends to be less efficient towards the problems with strong heterogeneity. 
As a result, under strong data heterogeneity, the local models diverge, and the training process stalls. Early work introduced similarity assumptions across clients, such as identical data or bounded gradients \citep{koloskova2020unified,gorbunov2021local}. In practice, most of these assumptions are infeasible \citep[Table 3]{karimireddy2020scaffold}.
The issue was tackled by the \texttt{SCAFFOLD} algorithm \citep{karimireddy2020scaffold}. The authors introduced control variates that allow clients to maintain the correct update direction during local steps. This method was later extended with compression in \texttt{SCAFCOM} \citep{huang2024stochastic}.

\vspace{5mm}

\subsection{Spectral Methods in FL}
\citet{jordan2024muon} introduced \texttt{Muon}, a matrix-structure-aware method. In a nutshell, the algorithm orthogonalizes the updates via SVD, normalizing the gradients and thereby stabilizing training. Notably, this orthogonalization can be viewed as a Linear Minimization Oracle operator over the gradients in the spectral matrix norm \citep{bernstein2024old,pethick2025training}. Moreover, \citet{kim2026convergence} established the convergence of \texttt{Muon} under inexact SVD computed by the Newton-Schulz iteration, further motivating its practical adoption.
These theoretical insights and its practical performance have drawn considerable attention, motivating a growing body of work that studies and extends the method \citep{liu2025muon,qian2025muon,ahn2025dion,si2025adamuon,liu2026muon,zhang2026teon}.

Some recent papers employed spectral methods in a distributed manner \citep{ahn2025dion,therien2025muloco}, but only in the homogeneous setting. As discussed above, this assumption is unrealistic.
The first work to fully address this was \citep{takezawa2025fedmuon}. The authors use a general scheme of \texttt{SCAFFOLD} \citep{karimireddy2020scaffold}. Their \texttt{FedMuon} replaces the plain gradients in the local steps of \texttt{SCAFFOLD} with LMO updates. However, the scheme assumes the transmission of full model updates and control variates, which incurs communication overhead.
%The authors achieved a $\mathcal{O}(\nicefrac{1}{R^{\nicefrac{1}{2}}})$ convergence rate. 
Subsequently, compression was applied to \texttt{Muon} updates via \texttt{EF21} \citep{gruntkowska2025error} and the VR technique of \texttt{VR-MARINA} \citep{qian2026communication}. Crucially, both works forgo local steps, accessing the LMO once per round at the server -- yet local steps are a pillar of communication-efficient FL. The communication bottleneck for \texttt{Muon}-like methods, therefore, remains unresolved.
%To the best of our knowledge, our paper is the first in the literature to analyse local steps and contractive compression in FL setting from the standpoint of LMO updates.

\section{Contribution}\label{sec:contribution}
We summarize the contribution of our work in this section.
\begin{itemize}
\item We propose \texttt{FedCALM-D} (\textbf{Fed}erated \textbf{C}ompressed \textbf{A}nd \textbf{L}ocal \textbf{M}uon — \textbf{D}eterministic), a novel scheme that integrates both compression and local steps into \texttt{Muon}-like LMO updates in the deterministic FL setting. This requires more than a direct application of existing tools: we adapt the compression mechanism to the structure of the LMO update.
\item We extend the scheme to the stochastic setting and obtain our main algorithm, \texttt{FedCALM}. 
\item We establish rigorous convergence guarantees for both methods in the non-convex setting under mild assumptions. Crucially, our analysis is heterogeneity-free.
\item We validate our method on an image-classification benchmark (\textit{ResNet-18} on \textit{CIFAR-10}) and further on a language task (\textit{LoRA} fine-tuning of \textit{RoBERTa} on the \textit{GLUE} benchmark), confirming its practical advantages across modalities. In terms of communicated bits, \texttt{FedCALM} outperforms all baselines in both convergence speed and final accuracy.
\end{itemize}

\section{Main Results}\label{sec:main}
\subsection{Preliminaries}

First, we introduce the relevant notation for the equations in the following sections. 

\paragraph{Matrix norms.}
We consider $\|\cdot\|$ as the spectral matrix norm in $\mathbb{R}^{d_1\times d_2}$, $\|\cdot\|_*$ denotes its conjugate nuclear norm, and $\|\cdot\|_F$ is Frobenius norm.
%With the slight abuse of notation, we consider the matrices $\mathbb{R}^{d_1\times d_2}$ and vectors to be the same objects. We interpret obects as matrices only for the purposes of caclulating the appropriate norms, and otherwise treat those as vectors. 
The following inequalities hold for introduced norms with $\rho = \sqrt{\min(d_1,d_2)}$ \citep[Proposition 2]{kim2026convergence}:
\begin{align}\label{eq:norms}
    \|x\| \leq \|x\|_F \leq \|x\|_*\leq \rho\|x\|_F,
\end{align}
\paragraph{Compression.}
Since the cornerstone of our paper is compression, we define the corresponding operator.
\begin{definition}
    Operator $C: \mathbb{R}^{d_1\times d_2}\rightarrow \mathbb{R}^{d_1\times d_2}$ is a contractive compressor if there exists a constant $\beta_1>1$ such that
    \begin{align}\label{eq:compres}
        \mathbb{E}[\|C(x) - x\|] \leq& \left(1-\frac1{\beta_1}\right)\|x\|.
    \end{align}
\end{definition}
\begin{remark}
    We also assume that $C$ is a contractive compressor in the conjugate norm if there exists a constant $\beta_2>1$ such that
    \begin{align*}
        \mathbb{E}[\|C(x) - x\|_*] \leq& \left(1-\frac1{\beta_2}\right)\|x\|_*.
    \end{align*}
\end{remark}
We denote $\beta = \max(\beta_1,\beta_2)$. An example of a contractive compressor in matrix norm is truncated SVD:
$$ C(X) = U^T\Sigma_kV, $$
where $X=U^T\Sigma V$ is a full SVD decomposition, and $\Sigma_k$ retains $k$ largest singular values $\sigma_1\dots \sigma_k$ of $X$. The compressor satisfies the introduced definitions with $\beta_1 = \frac{\sigma_1}{\sigma_1-\sigma_{k+1}}, $ and $\beta_2 = \frac{\sum_{j=1}^r\sigma_j}{\sum_{j=1}^k\sigma_j}$.
%The canonical example of a contractive compressor is \textit{Top}$K$. It retains the $K$ entries of largest absolute value and zeroes out the rest. Although classically defined for vectors, the operator applies analogously to matrices, keeping the $K$ largest-magnitude entries of $x \in \mathbb{R}^{d_1 \times d_2}$.
\paragraph{LMO.} 
The spectral optimization is performed via the Linear Minimization Oracle (LMO). It refers to the solution of the following subproblem:
\begin{align*}
    \operatorname{lmo}(x) = \arg\min_{y\in\mathbb{R}^{d_1\times d_2},\|y\|\leq 1} \langle x,y\rangle.
\end{align*}
By definition, for any vector $u$ we have 
$\|\operatorname{lmo}(u)\|\leq1$.
We refer to this equation as the unit-length of LMO.
The other distinctive feature of LMO is described in Lemma \ref{lem:lmo}:
\begin{align*}
    \langle x, \operatorname{lmo}(x)\rangle = -\|x\|_*.
\end{align*} The proof can be found in the appendix.
\paragraph{Newton-Schultz iterations.}
Since the exact calculation of LMO requires costly SVD, in practice, an approximation of the operator is calculated via the Newton-Schulz iterations \citep{kim2026convergence}.
\paragraph{Algorithmic notation.}
Next, we explain the notation of the algorithmic variables. 
$x^r\in\mathbb{R}^{d_1\times d_2}$ denotes the global model weights at the start of round $r$, $M$ is the total number of clients. $\eta$ is the stepsize, and $ y_i^r$ is the full update for client $i$ after $K$ local steps. The local and global control variates originating from \texttt{SCAFFOLD} are refered to as $c_i^r$ and $c^r$, respectively.
\paragraph{Assumptions.}
In our proofs, we assume the $L$-smoothness of the objective, following the LMO analysis in prior work \citep{pethick2025training,riabinin2025gluon}.
\begin{assumption}\label{ass:smooth}
    The objective $f_i$ for each client $i$ in $L_*$-smooth: for any $x,y\in\mathbb{R}^{d_1\times d_2}$,
    \begin{align*}
        \|\nabla f_i(x) - \nabla f_i(y)\|_* \leq L_* \|x - y\|.
    \end{align*}
\end{assumption}
The loss landscapes of neural networks are typically highly non-convex \citep{cybenko1989approximation}. As we target practical, real-world scenarios, our analysis focuses on establishing convergence in the non-convex setting.
\begin{assumption}\label{ass:non-convex}
    The function $f$ is non-convex with at least one global minimum: there exists may be not unique $x^*$ such that
        \begin{align*}
            f(x^*) = \inf\limits_{x \in \mathbb{R}^d} f(x) > -\infty.
        \end{align*}
\end{assumption}
% \begin{remark}
%     The two inequalities above are not stronger than any one of them alone. Using the norm equivalence \ref{eq:norms}, one can be derived from another:
%     $$\frac{c_*}{C}L_2 \leq L_*\leq \frac{C_*}cL_2$$
%     The two separate inequalities are assumed for convinience and clarity.
% \end{remark}
\paragraph{Communication.}
Finally, we introduce two notions of communication efficiency.
\begin{definition}
The number of communication rounds (round complexity) is the number of times the server communicates with the clients.
\end{definition}
This notion disregards which clients take part in a round and the sizes of the transmitted messages.
\begin{definition}\label{def:cc}
Communication time (time complexity) is the total time of communication with all clients, assuming their channels are of equal quality.
\end{definition}
Unlike the first, this notion accounts for the sizes of the messages. Unless stated otherwise, we adopt Definition \ref{def:cc} as our measure of communication complexity, since it captures the practical advantages of compression: compression reduces the transmission time of each round, which in turn accelerates convergence. Since we assume all channels are equal, communication time is proportional to the total number of transmitted bits; we therefore report communicated bits in our experiments as a direct proxy for time complexity.

\subsection{Motivation}\label{sec:motivation}
This section explains the design choices of our algorithm. First, we consider the deterministic setting. 

\textbf{Local drift.} As discussed in Section \ref{sec:related_works}, current federated methods with LMO updates lack one of two FL techniques: compression \citep{takezawa2025fedmuon} or local steps \citep{gruntkowska2025error,qian2026communication}. We aim to implement both in the same scheme. Since we assume no heterogeneity bound, local steps create client drift and spoil the convergence \citep[Figure 1]{karimireddy2020scaffold}. To overcome this, we consider control variates of \texttt{SCAFFOLD}. 
% The idea behind these is simple. The control variates for each client approximate respective gradients:
% \begin{equation}\label{eq:c_approx}
%     c_i^r \approx \nabla f_i(x^r),\ c^r = \frac1M\sum_{i=1}^Mc_i^r \approx \nabla f(x^r).
% \end{equation}
% This approximation holds during the local training because of the LMO unit-length feature. The client drift is limited over one round (Lemma \ref{lem:dist}):
% \begin{align*}
%     \|x_i^{r,K} - x_i^{r,0}\| = \mathcal{O}(\eta K).
% \end{align*} 
% Combined with smoothness (Assumption \ref{ass:smooth}), this leads to the bounded gradient drift:
% \begin{align}\label{eq:grad_drift1}
%     \|\nabla f_i(x^{r,k}_i) - \nabla f_i(x_i^{r,0})\|_* =\mathcal{O}( L_* \eta K).
% \end{align}
% Thus, clients maintain alignment between the aggregations and avoid divergence:
% \begin{align*}
%     x^{r,k+1}_i = x^{r,k}_i + \eta \operatorname{lmo}(\nabla f_i(x_i^{r,k}) - c_i^r + c^r) 
%     \approx x^{r,k}_i + \eta \operatorname{lmo}(\nabla f(x^r)).
% \end{align*} 
% The bound makes the scheme viable as long as we can preserve $c_i^r \approx \nabla f_i(x^r)$ approximation.

\textbf{Two compressions.} With compression, the main challenge arises: \textbf{how to properly implement compressors in the scheme?} We need to communicate both the local updates $y_i^r=x_i^{r,K}-x_i^{r,0}$ and control variates $c_i^r$. \citet{huang2024stochastic} faced a similar challenge in the design of \texttt{SCAFCOM}. The authors managed to combine the two messages into one and apply \texttt{EF21}. The update for the control variate $\Delta c_i^r$ can be expicitly derived from the update $y_i^r$ for the model.

However, this unification is impossible in LMO setting. Client drift has an LMO structure and \textit{provides no information about the scale of the gradient} due to the unit length of LMO, which is essential for the control variate update. \textit{Thus, the client drift and the control-variate update cannot be encoded in a single message and call for two distinct compression steps.}

\textbf{Model update.} Let us discuss the compression for model update $y_i^r$. For \texttt{SCAFCOM}, \texttt{EF21} was fitting due to the diminishing gradients. As the model approaches the optimal point, $\|\nabla f(x^r)\|\rightarrow 0$ and $\|y_i^r\|\rightarrow 0$ lead to smaller errors. This is not the case in the spectral setting. 
Since the update has LMO nature, it therefore always has the size $\mathcal{O}(\eta K)$. The direction is what changes across rounds. Even near the optimal point, $\|y_i^{r+1}-y_i^r\|=\mathcal{O}(\eta K)$. \texttt{EF21} scheme would introduce bias at each round, and of the same magnitude as the full vector $y_i^r$. This creates non-vanishing error and prevents convergence. 
\begin{algorithm}{\texttt{FedCALM-D}}\label{alg:fedcalm-d}
\begin{algorithmic}[1]
\State \textbf{Input:} initial point $x^0 \in \mathbb{R}^d$, $c^{0} = \nabla f(x^0)$, $e_i^{0} = 0,c_i^0=\nabla f_i(x^0)$ for every client $i=1,\dots M$, number of rounds $R$, round length $K$
\State \textbf{Hyperparameters:} stepsize $\eta > 0$, number of Newton-Schulz steps $q$
\For{$r = 0, 1, \ldots, R-1$}
  \For{clients $i = 1, \ldots, M$}
    \State $x_i^{r,0} = x^r $ \label{line:fedcalm-d:init_x}
    \For{local step $k = 0,1, \ldots, K-1$}
      \State $x_i^{r,k+1} = x_i^{r,k} + \eta\texttt{NS}^q(\nabla f_i(x_i^{r,k}) - {c}_i^r + c^r) $ \label{line:fedcalm-d:local_step}
    \EndFor
    \State $y_i^{r+1} = x^{r,K}_i - x_i^{r,0} $ \label{line:fedcalm-d:update}
    \State $v_i^{r+1} = C(y_i^{r+1} + e_i^{r}) $ \label{line:fedcalm-d:compress_update}
    \State $e_i^{r+1} = e_i^r + y_i^{r+1} - v_i^{r+1} $ \label{line:fedcalm-d:error}
    \State Clients send $v_i^{r+1}$ to the server \label{line:fedcalm-d:collect_updates}
    \State $\delta_i^{r+1} = C(\nabla f_i(x_i^{r,K}) - c_i^{r}) $\label{line:fedcalm-d:ef21_1}
    \State $c_i^{r+1} = c_i^{r} + \delta_i^{r+1} $  \label{line:fedcalm-d:ef21_2}
    \State Clients send $\delta_i^{r+1}$ to the server \label{line:fedcalm-d:collect_c}
  \EndFor
  \State Server updates $x^{r+1} = x^{r} + \frac{1}{M}\sum_{i=1}^Mv_i^{r+1} $ \label{line:fedcalm-d:global_update}
  \State Server updates $c^{r+1} = c^r + \frac1M\sum_{i=1}^M\delta_i^{r+1} $ \label{line:fedcalm-d:update_c}
  \State Server sends $x^{r+1}, c^{r+1}$ to all clients \label{line:fedcalm-d:send_new_point}
\EndFor
\State \Return $\frac1R\sum_{r=0}^{R-1}x^r$
\end{algorithmic}
\end{algorithm}

\textbf{Control variates.} The motivation for the implementation of control variates instantly suggests the error compensation scheme to use ~--~ \texttt{EF21}. Contrary to the model update, the gradient changes slowly across the rounds due to nature of the model update, and \texttt{EF21} is designed to track these small changes. 
Moreover, for $c_i^r$, the natural transition to the virtual sequence of \texttt{EF} is impossible. 

This yields \texttt{FedCALM-D} (\textbf{Fed}erated \textbf{C}ompressed \textbf{A}nd \textbf{L}ocal \textbf{Muon-Deterministic}\textbf{D}) given in Algorithm \ref{alg:fedcalm-d}.
\subsection{Deterministic Algorithm}
We describe the \texttt{FedCALM-D} algorithm in detail. At the start, the control variates are initialized to the initial gradients. A full, uncompressed gradient communication is required at this stage, but not at any other point; this one-time transmission does not affect the convergence rate. Moreover, a similar heavy initialization is typical for methods with compression \citep{richtarik2021ef21,huang2024stochastic}. At each round, all clients perform local training in parallel, starting from the shared point $x^r$ received from the server (Line \ref{line:fedcalm-d:init_x}). 
The local steps follow the discussed structure of \texttt{SCAFCOM} (Line \ref{line:fedcalm-d:local_step}).
% \begin{align*}
%     x_i^{r,k+1} = x_i^{r,k} + \eta\texttt{NS}^q(\nabla f_i(x_i^{r,k}) - {c}_i^r + c^r).
% \end{align*}
Here, \texttt{NS}$^q(\cdot)$ denotes $q$ steps of the Newton-Schulz algorithm \citep{kim2026convergence}. As discussed in Section \ref{sec:motivation}, the clients maintain $c_i^r \approx \nabla f_i(x^r),~c^r\approx \nabla f(x^r)$, which provides $\nabla f_i(x_i^{r,k}) - {c}_i^r + c^r \approx \nabla f(x^r)$. For sufficiently small steps, this estimate is expected to hold. After $K$ local steps, the total client update $y_i^{r+1}$ is compressed and sent to the server, and the clients' errors are corrected via \texttt{EF} (Lines \ref{line:fedcalm-d:update}--\ref{line:fedcalm-d:collect_updates}). The control variates are in turn updated via \texttt{EF21} (Lines \ref{line:fedcalm-d:ef21_1}--\ref{line:fedcalm-d:collect_c}). 
The server collects the compressed updates $v_i^{r+1}$ and $\delta_i^{r+1}$ and utilizes those to compute the new point (Line \ref{line:fedcalm-d:global_update}) and the new global control variate $c^{r+1}$ (Line \ref{line:fedcalm-d:update_c}). The vectors are then and broadcasted to all clients at the end of the round (Line \ref{line:fedcalm-d:send_new_point}). 
 
% Here, \texttt{NS}$^q(\cdot)$ denotes $q$ steps of the Newton-Schulz algorithm \citep{kim2026convergence}. As discussed in Section \ref{sec:motivation}, the clients maintain $c_i^r \approx \nabla f_i(x^r),~c^r\approx \nabla f(x^r)$, which provides $\nabla f_i(x_i^{r,k}) - {c}_i^r + c^r \approx \nabla f(x^r)$. For sufficiently small steps, this estimate is expected to hold. After $K$ local steps, the total client update $y_i^{r+1}$ is compressed and sent to the server, and the clients' errors are corrected via \texttt{EF} (Lines \ref{line:fedcalm-d:update}--\ref{line:fedcalm-d:collect_updates}). The server then computes the new iterate as the averaged update (Line \ref{line:fedcalm-d:global_update}) and transmits it to the clients.
 
% After updating the server model $x^{r+1}$, a second stage of the round begins. \textit{This is a crucial algorithmic distinction of our method from \texttt{FedMuon} and \texttt{SCAFCOM}}. We update the control variates at the new point $x^{r+1}$ via \texttt{EF21} (Lines \ref{line:fedcalm-d:ef21_1}--\ref{line:fedcalm-d:collect_c}). The server computes the new global control variate $c^{r+1}$ (Line \ref{line:fedcalm-d:update_c}) and broadcasts it to all clients at the end of the round (Line \ref{line:fedcalm-d:send_c}). This separate aggregation allows the clients to begin the next round with up-to-date control variates.
 
Note that the setting is deterministic: we assume access to full gradients $\nabla f_i(x)$. In real-world applications, this makes \texttt{FedCALM-D} impractical, and that is by design: it serves as a theoretical milestone that converges provably under both compression and local steps. We now state the convergence results for Algorithm \ref{alg:fedcalm-d}. The proof of Theorem \ref{th:conv_fedcalm-d} can be found in the appendix.
\begin{theorem}\label{th:conv_fedcalm-d}
    Suppose Assumptions \ref{ass:smooth}, \ref{ass:non-convex} hold. Then Algorithm \ref{alg:fedcalm-d} to achieve $\varepsilon$-accuracy needs
    \begin{equation*}
     R = \mathcal{O}\left(\frac{\beta^2\Delta L_*}{\varepsilon^2(1-\varepsilon_q)^2}\right) \text{~~communication rounds,}
    \end{equation*}
    where $\displaystyle\varepsilon = \mathbb{E}\left[\frac1R\sum_{r=0}^{R-1}\|\nabla f(x^r)\|_*\right]$, $\Delta=f(x^0)-{f(x^*)}$, and $\varepsilon_q$ is the Newton-Schulz steps error.
\end{theorem}

\textbf{Discussion.}  
Our asymptotic rate coincides with that of \texttt{SCAFCOM} when only the non-stochastic terms are considered \citep[Theorem 3]{huang2024stochastic}:
\begin{align*}
    R = \mathcal{O}\left(\frac{\beta\Delta L_F}{\varepsilon^2}\right), \quad \text{where} \quad \varepsilon =\mathbb{E}\left[\frac1R\sum_{r=0}^{R-1}\|\nabla f(x^r)\|_F\right].
\end{align*}
Thus, in the deterministic setting, Algorithm \ref{alg:fedcalm-d} attains the same asymptotic rate in $\varepsilon$ under both compression and local steps. Our bound nonetheless differs from that of \texttt{SCAFCOM}.
 
First, we state convergence in the nuclear norm, which upper-bounds the Frobenius norm \eqref{eq:norms}. Consequently, driving $\|\nabla f(x^r)\|_*$ below $\varepsilon$ automatically drives $\|\nabla f(x^r)\|_F$ below the same $\varepsilon$, so our stationarity criterion is strictly stronger. Empirically, the Hessian spectrum of deep networks is approximately sparse, concentrated on a few large eigenvalues \citep{papyan2018full}, which keeps the two smoothness constants comparable, $L_* \approx L_F$. In this regime the two bounds carry essentially the same constants, and Algorithm \ref{alg:fedcalm-d} therefore certifies a stronger criterion at the same asymptotic cost --- mirroring the advantage of \texttt{Muon} over Euclidean methods.

Inexact LMO analysis also introduces the Newton-Schulz error $\varepsilon_q$, which is thoroughly analysed by \citet{kim2026convergence}. The authors state it exponentially converges to zero, making it negligible.
 
Further, our bound carries a higher power of $\beta$. This stems from our compression procedure: the output of the first compressor (Line \ref{line:fedcalm-d:compress_update}) enters, indirectly, the argument of the second (Line \ref{line:fedcalm-d:ef21_1}). Such nested compression is known to introduce an additional factor of $\beta$ \citep{gao2025accelerated}:
\begin{align*}
    c_i^{r+1} = c_i^r + C\left[ \nabla f_i\!\left( x^r + \frac1M\sum_{i=1}^M C(y_i^{r+1}+e_i^r) \right) - c_i^r \right].
\end{align*}

\subsection{Stochastic Algorithm}
We now turn to the stochastic setting, where computing full gradients is infeasible and only unbiased estimates are available. We denote a stochastic gradient by $g_i(x,\xi)$, where $i$ indexes the client and $\xi$ is a random sample from its local dataset. We adopt the standard assumption of bounded variance.
\begin{assumption}\label{ass:stoch}
    An unbiased stochastic gradient is available for client $i$ with bounded variance:
    \begin{align*}
        &\mathbb{E}[g_i(x,\xi)]=\nabla f_i(x),
        \\&\mathbb{E}\|g_i(x,\xi) - \nabla f_i(x)\|_2^2 \leq \sigma^2.
    \end{align*}
\end{assumption}
\begin{algorithm}{\texttt{FedCALM}}\label{alg:fedcalm}
\begin{algorithmic}[1]
\State \textbf{Input:} initial point $x^0 \in \mathbb{R}^d$, $c^{0} = \nabla f(x^0)$, $e_i^{0} = 0,c_i^0=\nabla f_i(x^0)$ for every client $i=1,\dots M$, number of rounds $R$, round length $K$
\State \textbf{Hyperparameters:} stepsize $\eta > 0$, momentum $\alpha \in [0;1]$, number of Newton-Schulz steps $q$
\For{$r = 0, 1, \ldots, R-1$}
  \For{clients $i = 1, \ldots, M$}
    \State $x_i^{r,0} = x^r, m_i^{r,0}=m_i^{r-1,K} $ \label{line:fedcalm:init_round}
    \For{local step $k = 0,1, \ldots, K-1$}
      \State $m_i^{r,k+1} = (1-\alpha)m_i^{r,k} + \alpha g_i(x_i^{r,k},\xi_i^{r,k}) $\label{line:fedcalm:local_momentum}
      \State $x_i^{r,k+1} = x_i^{r,k} + \eta\texttt{NS}^q(\nabla f_i(x_i^{r,k}) - {c}_i^r + c^r) $ \label{line:fedcalm:local_step}
    \EndFor
    \State $y_i^{r+1} = x^{r,K}_i - x_i^{r,0} $ \label{line:fedcalm:update}
    \State $v_i^{r+1} = C(y_i^{r+1} + e_i^{r}) $ \label{line:fedcalm:compress_update}
    \State $e_i^{r+1} = e_i^r + y_i^{r+1} - v_i^{r+1} $ \label{line:fedcalm:error}
    \State Clients send $v_i^{r+1}$ to the server \label{line:fedcalm:collect_updates}
    \State $\delta_i^{r+1} = C(m_i^{r,K} - c_i^{r}) $\label{line:fedcalm:ef21_1}
    \State $c_i^{r+1} = c_i^{r} + \delta_i^{r+1} $  \label{line:fedcalm:ef21_2}
    \State Clients send $\delta_i^{r+1}$ to the server \label{line:fedcalm:collect_c}
  \EndFor
  \State Server updates $x^{r+1} = x^{r} + \frac{1}{M}\sum_{i=1}^Mv_i^{r+1} $ \label{line:fedcalm:global_update}
  \State Server updates $c^{r+1} = c^r + \frac1M\sum_{i=1}^M\delta_i^{r+1} $ \label{line:fedcalm:update_c}
  \State Server sends $x^{r+1}, c^{r+1}$ to all clients \label{line:fedcalm:send_new_point}
\EndFor
\State \Return $\frac1R\sum_{r=0}^{R-1}x^r$
\end{algorithmic}
\end{algorithm}

We modify Algorithm \ref{alg:fedcalm-d} by adding a momentum, which leads to our main algorithm \texttt{FedCALM} (\textbf{Fed}erated \textbf{C}ompressed \textbf{A}nd \textbf{L}ocal \textbf{M}uon), given in Algorithm \ref{alg:fedcalm}.
\texttt{FedCALM} is a generalization of \texttt{FedCALM-D} with the momentum. The momentum in the local steps (Line \ref{line:fedcalm:local_momentum}) mitigates the variance within a round; this update rule coincides with that of \texttt{FedMuon}. The introduced momentum becomes the target for control variate update (Line \ref{line:fedcalm:ef21_1}). Notably, if one sets $\alpha = 1$ in Algorithm \ref{alg:fedcalm}, that would directly yield Algorithm \ref{alg:fedcalm-d}. 
%Second, the \texttt{EF21} procedure is substantially modified: we introduce a gradient surrogate $s_i^r$, updated with VR as in \citet{qian2025muon}, and use it as the target for \texttt{EF21} in place of the true gradient $\nabla f_i(x^{r+1})$ used in \texttt{FedCALM-D}. Notably, \texttt{FedCALM} is a strict generalization: setting $\theta=\alpha=1$ recovers Algorithm \ref{alg:fedcalm-d}.

\begin{theorem}\label{th:conv_fedcalm}
    Suppose Assumptions \ref{ass:smooth}, \ref{ass:non-convex}, and \ref{ass:stoch} hold. Then Algorithm \ref{alg:fedcalm} to achieve $\varepsilon$-accuracy needs
    \begin{equation*}
     R = \mathcal{O}\left(\frac{{\Delta L_* \beta^2}}{\varepsilon^2(1-\varepsilon_q)^2} + \frac{\Delta L_*\rho^2\sigma^2\beta}{M\varepsilon^4(1-\varepsilon_q)^4}\right)\text{~communication rounds,}
    \end{equation*}
    where $\displaystyle\varepsilon = \mathbb{E}\left[\frac{1}{R}\sum_{r=0}^{R-1}\|\nabla f(x^r)\|_*\right]$, $\Delta=f(x^0)-{f(x^*)}$, and $\varepsilon_q$ is of Newton-Schulz steps error.
\end{theorem}
The proof of Theorem \ref{th:conv_fedcalm} can be found in the appendix.

\textbf{Discussion.} In the stochastic setting, Algorithm \ref{alg:fedcalm} attains the same asymptotic rate as \texttt{SCAFCOM}, showing that it successfully assimilates both compression and local steps into the spectral FL. Our bound matches that of \texttt{SCAFCOM} in the two leading terms \citep[Theorem 3]{huang2024stochastic}:
\begin{equation*}
     R = \mathcal{O}\left(\frac{\beta\Delta L_F}{\varepsilon^2} + \frac{\Delta L_F\sigma^2}{MK\varepsilon^4}\right).
\end{equation*}
As in the deterministic case, our criterion is stronger: since the analysis is carried out in the spectral norm, we control a larger $\varepsilon$ over the same asymptotic number of rounds. As before, both terms carry an additional factor of $\beta$ due to nested compression. The stochastic term, in turn, carries an extra factor of $\rho^2$, which stems from the norm inequalities (\ref{eq:norms}) and reflects the conversion between the Frobenius and nuclear norms. Up to this factor, the stochastic term coincides with that of \texttt{SCAFCOM}.

\section{Experiments} \label{sec:experiments}
%To justify the relevance of the proposed approach, we perform experimental validation. 
\paragraph{Setup.}
The task at hand is image classification on the \textit{CIFAR-10} dataset \citep{CIFAR}. The chosen architecture is \textit{ResNet-18} \citep{he2016deep} ~--~ the standard benchmark for optimization problems \citep{meng2019frame}.
We examine the following federated setting. The training data is split among $M=10$ clients with two different distributions.  For more details, refer to the Appendix \ref{app:experiments}.
%\\\textit{Pathological distribution.} Each client has access to a random number of objects from each class, varying from $0.9$ to $1.1$ compared to uniform distribution.
%\\\textit{Heterogeneous distribution.} First $4$ clients have first $4$ classes split evenly amongst them, and last $6$ clients have remaining $6$ classes uniformly split. This is a severely strong heterogeneity, as clients do not have access to some of the classes.

\textbf{Baselines.} We compare our approach to the competitors discussed in Section \ref{sec:related_works}: \texttt{SCAFCOM} \citep{huang2024stochastic}, \texttt{FedMuon} \citep{takezawa2025fedmuon}, and \texttt{EF21-Muon} \citep{gruntkowska2025error}. We also implement \texttt{LocalSGD} \citep{stich2018local} as a simple FL benchmark without compression.

\textbf{Compression.} We choose \textit{TopK\%} sparsification as a compression operator. This means the output vector retains a fixed fraction of entries with the largest absolute values. \textit{Top1\%} and \textit{Top10\%} compression rates are explored in this study. \textit{Top100\%} compression means communicating the full message, and is implmented for methods that do not account for compression (\texttt{FedMuon}, \texttt{LocalSGD}).

\textbf{Communication details.} To demonstrate the practical advantages of compression, we present the plots in terms of the communicated bits. Under assumption of a stable chanell, this is equivalent to the total communication time (Definition \ref{def:cc}). In Table 1 we state the total bits transmitted in a single round for each method per client. Notably, \texttt{FedCALM} communicates two separate messages over the round and thus doubles the transmission size compared to other schemes with compression. Hence, \texttt{FedCALM} achives twice as much total communicated bits as \texttt{SCAFCOM} and \texttt{EF21-Muon}, while being trained for the same number of rounds (100 rounds).
\subsection{Empirical Results}
\begin{table}[ht]
\centering
\begin{tabular}{*6c}
\toprule
\multirow{2}{*}{Method} & \multirow{2}{*}{Compression} & Model & Accuracy  & Accuracy & Accuracy  \\
 & &size  &(\textit{CIFAR-10, Het.}) &(\textit{CIFAR-10, Path.}) & (\textit{SST-2}) \\
\midrule
\multirow{2}{*}{\texttt{FedCALM}} & \textit{Top1\%} & $\times$0.02 & 0.814 & 0.956 & 0.930 \\
   &  \textit{Top10\%} & $\times$0.2 & 0.825 & 0.959 & 0.928 \\
\multirow{2}{*}{\texttt{SCAFCOM}} & \textit{Top1\%} & $\times$0.01 & 0.460 & 0.904 & 0.933 \\
   &  \textit{Top10\%} & $\times$0.1 & 0.801 & 0.923 & 0.927 \\
\multirow{2}{*}{\texttt{EF21-Muon}} & \textit{Top1\%} & $\times$0.01 & 0.374 & 0.556 & 0.900 \\
   &  \textit{Top10\%} & $\times$0.1 & 0.601 & 0.746& 0.912 \\
{\texttt{FedMuon}} & \textit{Top100\%} & $\times$1 & 0.826 & 0.903 & 0.930 \\
{\texttt{LocalSGD}} & \textit{Top100\%} & $\times$1 & 0.811 & 0.886 & 0.927 \\
\bottomrule
\end{tabular}
\caption{Comparison of the methods performance. Model size designates the amount of transmitted information compared to the size of a single model. Last columns report best accuracy across training for all methods in three setups: \textit{CIFAR-10} with \textit{Heterogeneous} and \textit{Pathological} distributions, and \textit{SST-2} with \textit{IID} distribution.}\label{tab:acc}
\end{table}

\begin{figure}[H]
  \centering
  \includegraphics[
    width=0.45\columnwidth,
    %height=0.20\textheight,
    %keepaspectratio
  ]{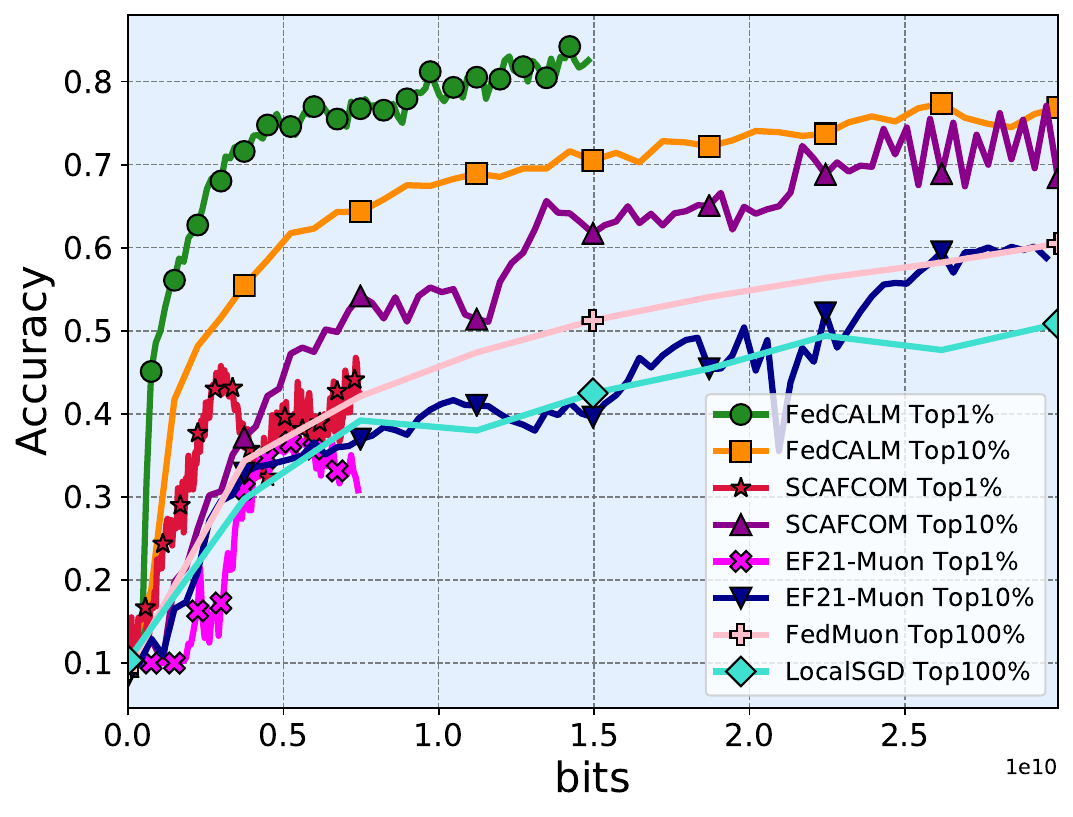}
  \includegraphics[
    width=0.45\columnwidth,
    %height=0.20\textheight,
    %keepaspectratio
  ]{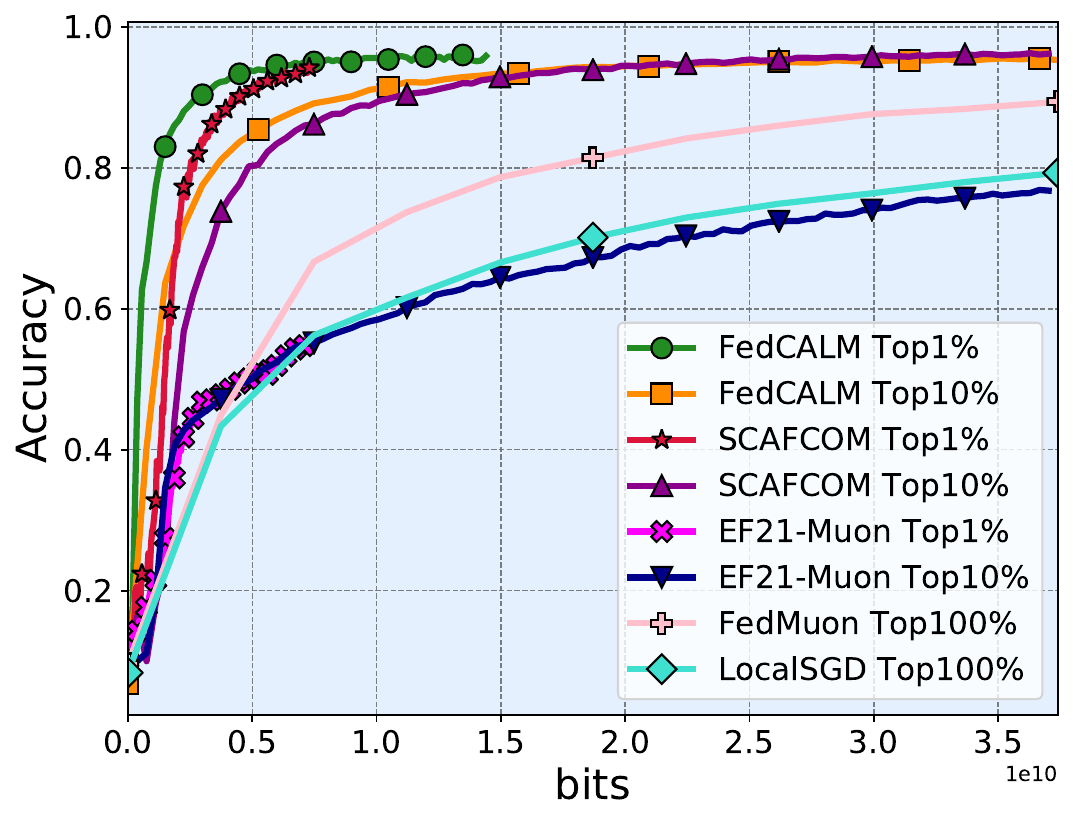}
  \caption{The test accuracy accross the training on \textit{Heterogeneous distibution} (left) and \textit{Pathological distribution} (right) for all methods. The horizontal axis designates the amount of communicated bits.}\label{fig:acc}
  %\vspace{-3mm}
\end{figure}
%The plots also demonstrate the advantages of the spectral approaches: \texttt{FedMuon} outperforms \texttt{FedAvg} and \texttt{FedCALM} surpasses \texttt{SCAFCOM}. 
%Notably, despite being a spectral method, \texttt{EF21-Muon} falls back. This is a consequence of the abcense of local steps which heavily facilitate the convergence.
\begin{wrapfigure}{r}{0.5\columnwidth}
  \centering
  \includegraphics[
    width=0.5\columnwidth,
    %height=0.20\textheight,
    %keepaspectratio
  ]{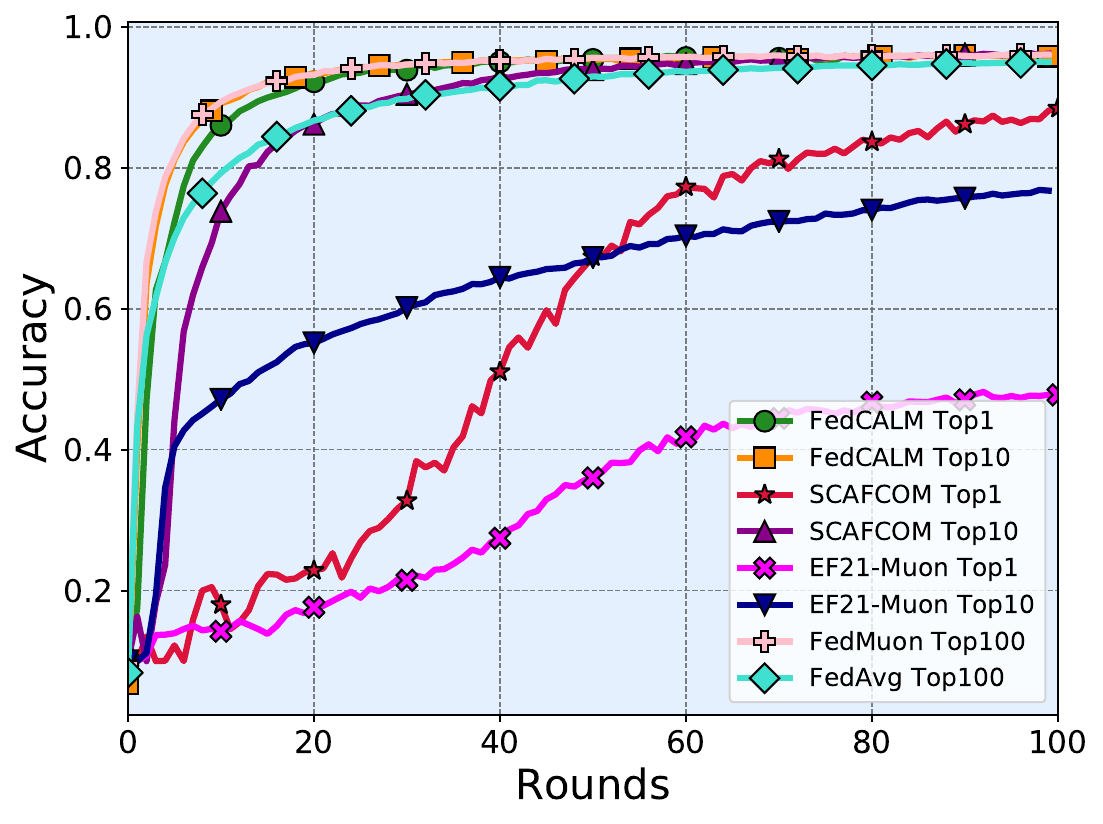}
  \caption{The test accuracy accross the training on \textit{Pathological distibution} for all methods over the communication rounds. That is, compression rates are not taken into account.}\label{fig:pathology_rounds}
  \vspace{-3mm}
\end{wrapfigure}
\paragraph{\textit{CIFAR-10} experiments.} As demonstrated by the plots (see Figure \ref{fig:acc}), proposed \texttt{FedCALM} method consistently outperforms all of the competitors in terms of communicated bits.
%\texttt{SCAFCOM}, \texttt{FedMuon}, \texttt{EF21-Muon} and \texttt{FedAvg}. 
Across two data distributions, \texttt{FedCALM} and \texttt{SCAFCOM} show the best performance by a large margin, evidently demonstrating the superiority of methods with both local steps and compression. Moreover, stronger compression almost directly translates to faster convergence. However, in \textit{Heterogeneous distribution}, \texttt{SCAFCOM} with \textit{Top1\%} compression performs worse than with the moderate compression \textit{Top10\%}. The reason for that may be the shear complexity of the problem: severe heterogeneity is combined with strong sparsification.

Figure \ref{fig:pathology_rounds} clearly demonstrates the advantages of compression. The plot shows that the performance of most methods is similar in terms of rounds, and there is almost no slowdown with compression applied.

Moreover, the domination of spectral methods -- \texttt{FedMuon}, \texttt{FedCALM} -- is apparent compared to their euclidean counterparts -- \texttt{LocalSGD}, \texttt{SCAFCOM}.
Notably, the \texttt{EF21-Muon} stays left behind. The poor performance is likely rooted in the absence of local steps.

\paragraph{LLM fine-tuning.} To further extend the scope of our validation, we conduct experiments with a LLM fine-tuning on \textit{GLUE} benchmark \citep{wang2018glue}.
%, and demonstrate the performance of \texttt{FedCALM} in the different modality, which matches our previous results. 
We report the best accuracy across training in Table \ref{tab:acc}. The results demonstrate the strong performance of proposed \texttt{FedCALM} in the different modality, which matches our previous experiments.
The details of the setup are presented in Appendix \ref{app:sst2}. 

\subsection{Network load} 
Moreover, we perform a careful study of network load and observe the theoretical advantages of Algorithm \ref{alg:fedcalm} in terms of communication uplink. The results are described in Appendix \ref{app:network_load}.

\section{Limitations}
First, we rely on smoothness of the objective. In practice, it is fulfilled only locally, however, it is classic in the field.
Second, the scope of our empirical validation is limited, and the future work can extend the scale of the experiments.
Lastly, we consider all clients to be available at every communication round, while this may not be the case in real-world scenarios.

\section{Conclusion}
In this paper, we address the communication bottleneck of FL by proposing a novel spectral method, \texttt{FedCALM-D} , which utilizes two key techniques of FL: local steps and communication compression. Crucially, we manage to perform two separate communications with proper error compensation mechanisms for both.
The scheme provably converges with arbitrary heterogeneity in the deterministic case.

Furthermore, we expand our approach with \texttt{FedCALM} -- the stochastic version of Algorithm \ref{alg:fedcalm-d}. We obtain heterogeneity-free convergence bounds for \texttt{FedCALM}, asymptotically matching the classic rates of \texttt{SCAFCOM} in the spectral setting.

We perform experimental validation on computer vision and language processing tasks and demonstrate the empirical advantages of \texttt{FedCALM} over standard baselines.

\end{mainpart}

\begin{appendixpart}
\allowdisplaybreaks
\section{Additional Experiments}\label{app:experiments}

\subsection{Language Model Fine-Tuning on SST-2}
\label{app:sst2}
\noindent\textbf{Setup.} 
We extend our experiemnts to a bigger scale with a NLP task. We use SST-2 problem from the GLUE benchmark \citep{wang2018glue}: binary sentiment classification of single sentences taken
from movie reviews. Sentences are tokenized with the RoBERTa tokenizer and truncated or padded to $128$ tokens.
The rest of the setting is described in details in section \ref{app:experimental details}.
The results of the experiment are presented in Table \ref{tab:result_glue}. Each of the methods achieved roughly 93\% accuracy in the totla runtime of 30 communication rounds. 

\begin{table}[ht]
\centering
\begin{tabular}{*4c}
\toprule
Method & Compression & Accuracy  \\
\midrule
\multirow{2}{*}{\texttt{FedCALM}} & \textit{Top1\%} & 0.930 \\
   &  \textit{Top10\%} & 0.928 \\
\multirow{2}{*}{\texttt{SCAFCOM}} & \textit{Top1\%} & 0.933 \\
   &  \textit{Top10\%} & 0.927\\
\multirow{2}{*}{\texttt{EF21-Muon}} & \textit{Top1\%} & 0.900 \\
   &  \textit{Top10\%} & 0.912 \\
{\texttt{FedMuon}} & \textit{Top100\%} & 0.930 \\
{\texttt{LocalSGD}} & \textit{Top100\%} & 0.927 \\
\bottomrule
\end{tabular}
\caption{The test accuracy accross the training on \textit{IID distibution} for all methods over $30$ communication
rounds in \textit{SST-2} experiments.}\label{tab:result_glue}
\end{table}

\subsection{Network Load}\label{app:network_load}
In the main part we described the communication efficiency of \texttt{FedCALM} in terms of accuracy over the rounds and transmitted bits. However, we left out the network load during the runs and did not compare the methods based on the communication time. In this section, we report the corresponding metrics. In particular, Table \ref{tab:resnet_network} shows the following statistics: the peak client GPU usage during local training, total uplink memory and aggregation time at the server. The table shows that while GPU is utilized similarly across methods, the uplink and aggregation time benefit strongly from the compression. \texttt{FedCALM} requires almost twice as much resources as \texttt{SCAFCOM} and \texttt{EF21-Muon} due to the two-stage structure of the method. However, this is alleviated by the faster convergence. Table \ref{tab:resnet_resources_to90} reports the total uplink time and memory untill the methods reach 90\% accuracy. The results demonstrate the acceleration provided by \texttt{FedCALM} in terms of both of the metrics. Another statistic that we provide is the average client memory usage over the training. Our method demonstrates almost the lowest costs, only rivaled by \texttt{FedMuon}. \texttt{EF21-Muon} is absent in Table \ref{tab:resnet_resources_to90} since it did not achieve target accuracy.

\begin{table}[ht]
\centering
\caption{Network statistics for \textit{CIFAR-10} experiments.}
\begin{tabular}{ccccccc}
\toprule
Method & Client GPU &Compression &  Uplink & Aggr. time \\
\midrule
\multirow{2}{*}{\texttt{FedCALM}} & \multirow{2}{*}{208.9 MB} & \textit{Top1\%}  & 8.94 MB & 171 ms\\
 &  &  \textit{Top10\%}  & 89.4 MB & 815 ms\\
\multirow{2}{*}{\texttt{SCAFCOM}} & \multirow{2}{*}{235.9 MB} & \textit{Top1\%} & 4.47 MB &85 ms\\
 &  &  \textit{Top10\%} & 44.7 MB & 407 ms\\
\multirow{2}{*}{\texttt{EF21-Muon}} & \multirow{2}{*}{181.3 MB}& \textit{Top1\%}  & 4.47 MB & 85 ms\\
 &  &  \textit{Top10\%} & 44.7 MB &407 ms\\
{\texttt{FedMuon}} & {225.1} MB & \textit{Top100\%} & 894 MB & 3.62 s \\
{\texttt{LocalSGD}} & 181.6 MB & \textit{Top100\%} & 447 MB & 3.62 s\\
\bottomrule
\end{tabular}\label{tab:resnet_network}
\end{table}

\begin{table}[ht]
\centering
\caption{Memory usage and total aggregation time to reach 90\% accuracy for \textit{CIFAR-10} experiments.}
\begin{tabular}{ccccc}
\toprule
Method &Compression & Uplink  & Time & Average memory \\
\midrule
\multirow{2}{*}{\texttt{FedCALM}} & \textit{Top1\%}  & 143 MB &  2.73 s & 7.87 MB \\
 &  \textit{Top10\%}  & 1072 MB & 9.78 s & 5.90 MB \\
\multirow{2}{*}{\texttt{SCAFCOM}}  & \textit{Top1\%} & 522 MB & 9.94 s & 57.6 MB \\
 & \textit{Top10\%} & 11803 MB & 11.8 s & 14.2 MB\\
{\texttt{FedMuon}}  & \textit{Top100\%} & 4917 MB & 39.2 s & 4.39 MB\\
{\texttt{LocalSGD}} & \textit{Top100\%} & 13857 MB & 110 s & 13.8 MB\\
\bottomrule
\end{tabular}\label{tab:resnet_resources_to90}
\end{table}

\noindent\textbf{Hardware.} We consider cross-silo setup for $M$=10 clients with uplink time of 100 Mb/s and 50 ms setup latency. Each client is leveraging a single CPU (Intel Xeon 2.20 GHz) and a single GPU (NVIDIA A100) for computation, and the server operates with one CPU. The experiments are implemented in Python using the PyTorch library. 
%The total runtime for all experiments is approximately 140 hours.

\noindent\textbf{Network load for SST-2.} Similarly to the experiments on \textit{CIFAR-10}, the advantages of compression are clearly visible when the costs are properly considered. In Table \ref{tab:glue_network}, we report client peak GPU usage, uplink memory and time. \texttt{FedCALM} continues to showcase low resource demand with an overhead compared to \texttt{SCAFCOM}, \texttt{EF21-Muon}. Table \ref{tab:glue_resources_to90} demonstrates the gains of our method over the competitors when the total costs are considered untill 90\% accuracy is achieved. This highlights the dominance of \texttt{FedCALM} among the methods: it requires the least uplink time and memory to solve the task, which signifies strong practical performance.

\begin{table}[ht]
\centering
\caption{ Network statistics for \textit{SST-2} experiments.}
\begin{tabular}{lcccc}
\hline
Method & Client memory & Compression & Uplink & Aggr. time \\
\hline
\multirow{2}{*}{\texttt{FedCALM}} & \multirow{2}{*}{2952.98 MB} & \textit{Top1\%} & 0.74 MB & 109 ms \\
 &  & \textit{Top10\%} & 7.39 MB & 641 ms \\
\multirow{2}{*}{\texttt{SCAFCOM}} & \multirow{2}{*}{2947.30 MB} & \textit{Top1\%} & 0.37 MB & 80 ms \\
 &  & \textit{Top10\%} & 3.70 MB & 346 ms \\
\multirow{2}{*}{\texttt{EF21-Muon}} & \multirow{2}{*}{2937.51 MB} & \textit{Top1\%} & 0.37 MB & 80 ms \\
 &  & \textit{Top10\%} & 3.70 MB & 346 ms \\
\multirow{1}{*}{\texttt{FedMuon}} & \multirow{1}{*}{2948.33 MB} & \textit{Top100\%} & 73.91 MB & 5.96 s \\
\multirow{1}{*}{\texttt{LocalSGD}} & \multirow{1}{*}{2937.95 MB} & \textit{Top100\%} & 36.96 MB & 3.01 s \\
\hline
\end{tabular}\label{tab:glue_network}
\end{table}

\begin{table}[ht]
\centering
\caption{Memory usage and total aggregation time to reach $91\%$ accuracy for \textit{SST-2} experiments.}
\begin{tabular}{lcccc}
\hline
Method & Compression & Uplink & Time & Average memory \\
\hline
\multirow{2}{*}{\texttt{FedCALM}} & \textit{Top1\%} & 4 MB & 546 ms & 615.20 MB \\
 & \textit{Top10\%} & 15 MB & 1.28 s & 250.48 MB \\
\multirow{2}{*}{\texttt{SCAFCOM}} & \textit{Top1\%} & 7 MB & 1.43 s & 2210.48 MB \\
 & \textit{Top10\%} & 22 MB & 2.07 s & 736.83 MB \\
\multirow{2}{*}{\texttt{EF21-Muon}} & \textit{Top1\%} & -- & -- & -- \\
 & \textit{Top10\%} & 89 MB & 8.30 s & 2937.51 MB \\
\multirow{1}{*}{\texttt{FedMuon}} & \textit{Top100\%} & 222 MB & 17.89 s & 368.54 MB \\
\multirow{1}{*}{\texttt{LocalSGD}} & \textit{Top100\%} & 111 MB & 9.02 s & 367.24 MB \\
\hline
\end{tabular}\label{tab:glue_resources_to90}
\end{table}
\newpage
\subsection{Ablation Study} 
\paragraph{Newton-Shultz steps.} To validate the stability of proposed \texttt{FedCALM}, we provide ablation over number of Newton-Schulz steps $q$. Figure \ref{fig:resnet_ablation} and Table \ref{tab:ablation_glue} demonstrate the robustness of our method to the choice of $q$. The performance gets worse only slightly if a single Newton-Schulz step is applied. This means that \texttt{FedCALM} is robust to the choice of $q$, which makes it more easily applicable in practice. 
\begin{figure}[ht]
  \centering
  % \includegraphics[
  %   width=0.4\columnwidth,
  %   %height=0.20\textheight,
  %   %keepaspectratio
  % ]{plots/fedcalm_theta_ablation_top10_test_acc.pdf}
  \includegraphics[
    width=0.4\columnwidth,
    %height=0.20\textheight,
    %keepaspectratio
  ]{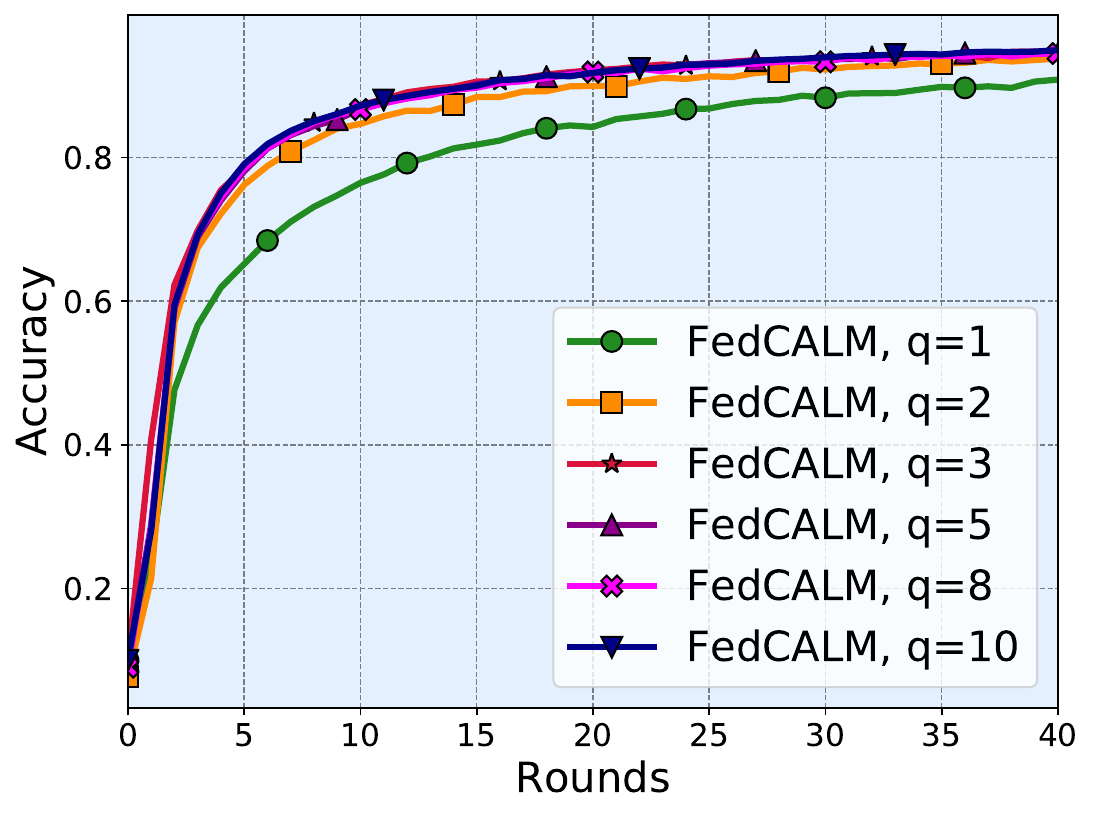}
  \caption{The ablation of \texttt{FedCALM} over the number of Newton-Schulz steps $q$ with \textit{CIFAR-10} experiments. The setup is the \textit{Pathological distribution} with \textit{Top10\%} compression rate.}\label{fig:resnet_ablation}
  %\vspace{-3mm}
\end{figure}

\begin{table}[ht]
\centering
\begin{tabular}{*3c}
\toprule
$q$ & Accuracy \\
\midrule
1 & 0.912 \\
2 & 0.926 \\
3 & 0.927 \\
5 & 0.928 \\
8 & 0.928 \\
10 & 0.929 \\
\bottomrule
\end{tabular}
\caption{The ablation of \texttt{FedCALM} over the number of Newton-Schulz steps $q$ with \textit{SST-2} experiments. The setup is the \textit{IID distribution} with \textit{Top10\%} compression rate during $20$ rounds.}
\label{tab:ablation_glue}
\end{table}

\paragraph{Compression.}
Further, we study the compression rates in the chosen norms. Fig. \ref{fig:compression_abl_het}, \ref{fig:compression_abl_path} reportes the explicit values for $\beta_1,\beta_2$ after performing the compression.
Evidently, the values of $\beta_1,\beta_2$ are upper-bounded by the corresponding compression rates (red lines) with a significant margin. At Top1\% compression, outliers are present at the beginning of the training. Nevertheless, those are  few in number ($\approx$2\%), and stay in at most one order of magnitude away from compression rate. Thus, the utilized compression can be considered practically contractive.

\begin{figure}[H]
    \centering
    \includegraphics[width=0.4\linewidth]{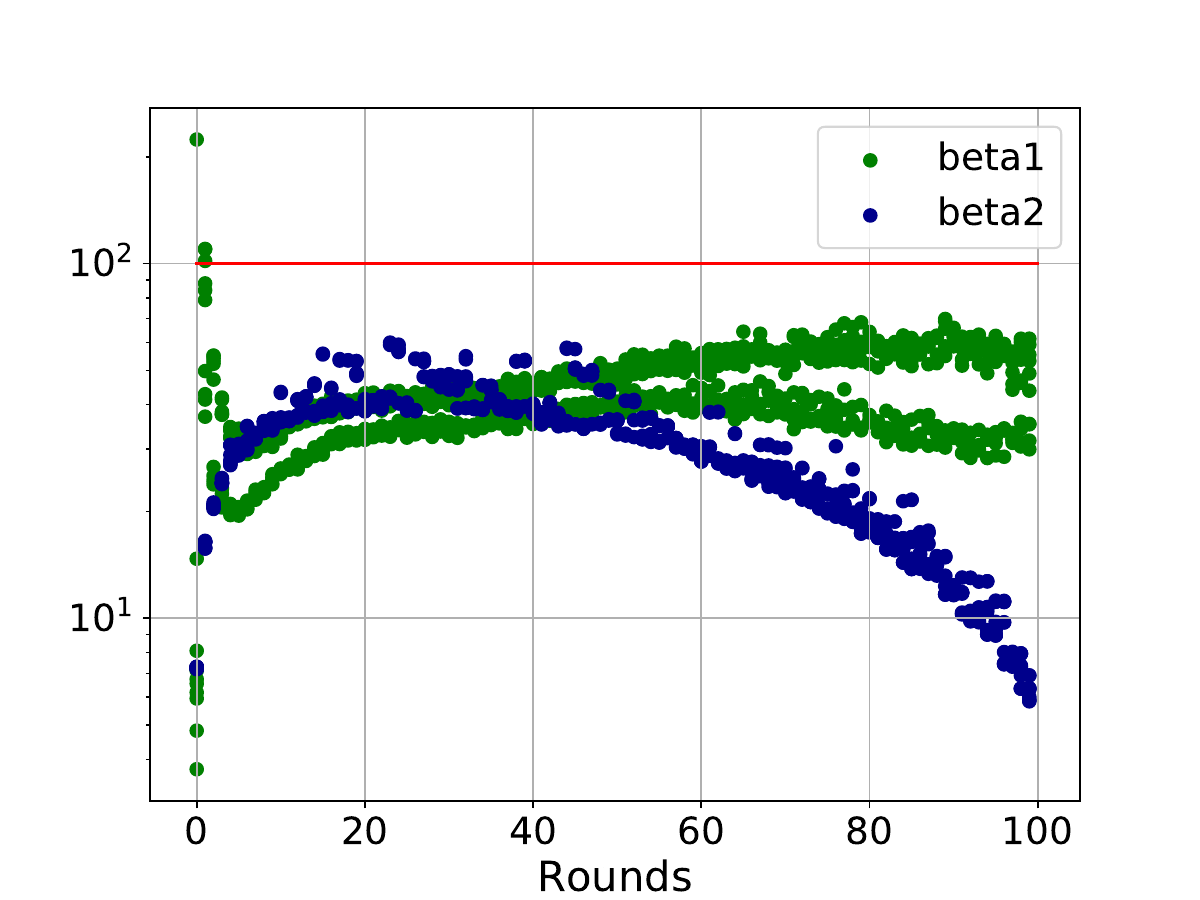}
    \includegraphics[width=0.4\linewidth]{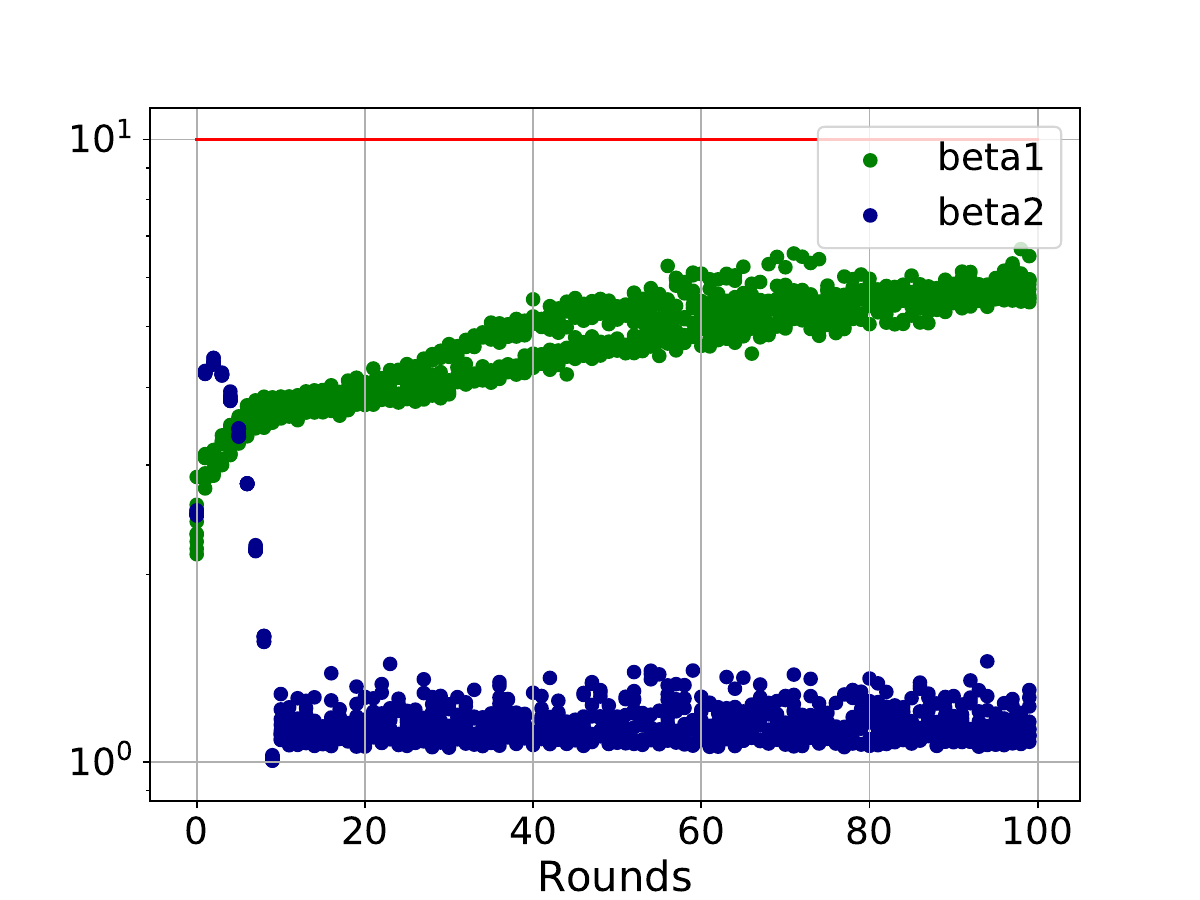}
    \caption{The empirical lower bounds for $\beta_1$ (green), $\beta_2$ (blue) for \textit{Heterogeneous} distribution. The red line marks the applied compression rates (Top1\% on the left and Top10\% on the right).}
    \label{fig:compression_abl_het}
\end{figure}
\begin{figure}[H]
    \centering
    \includegraphics[width=0.4\linewidth]{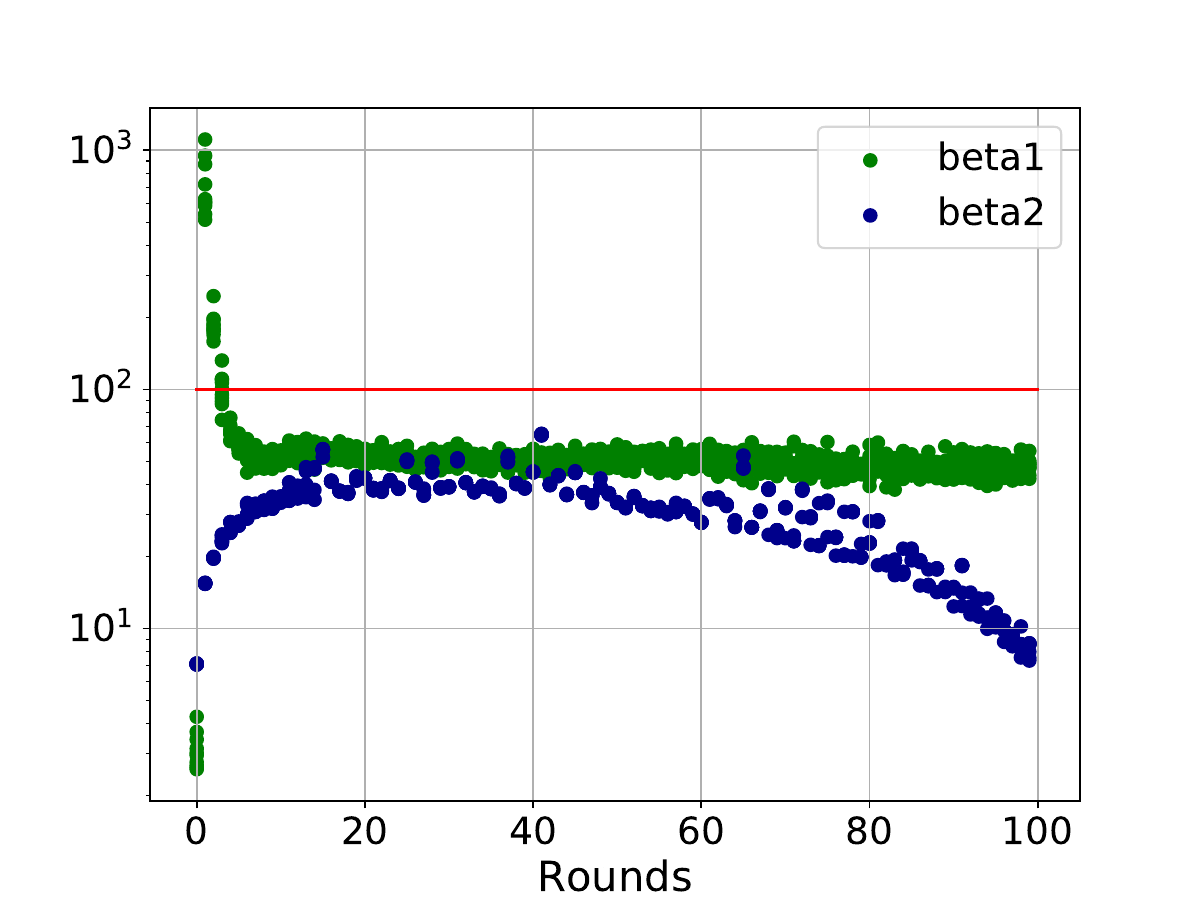}
    \includegraphics[width=0.4\linewidth]{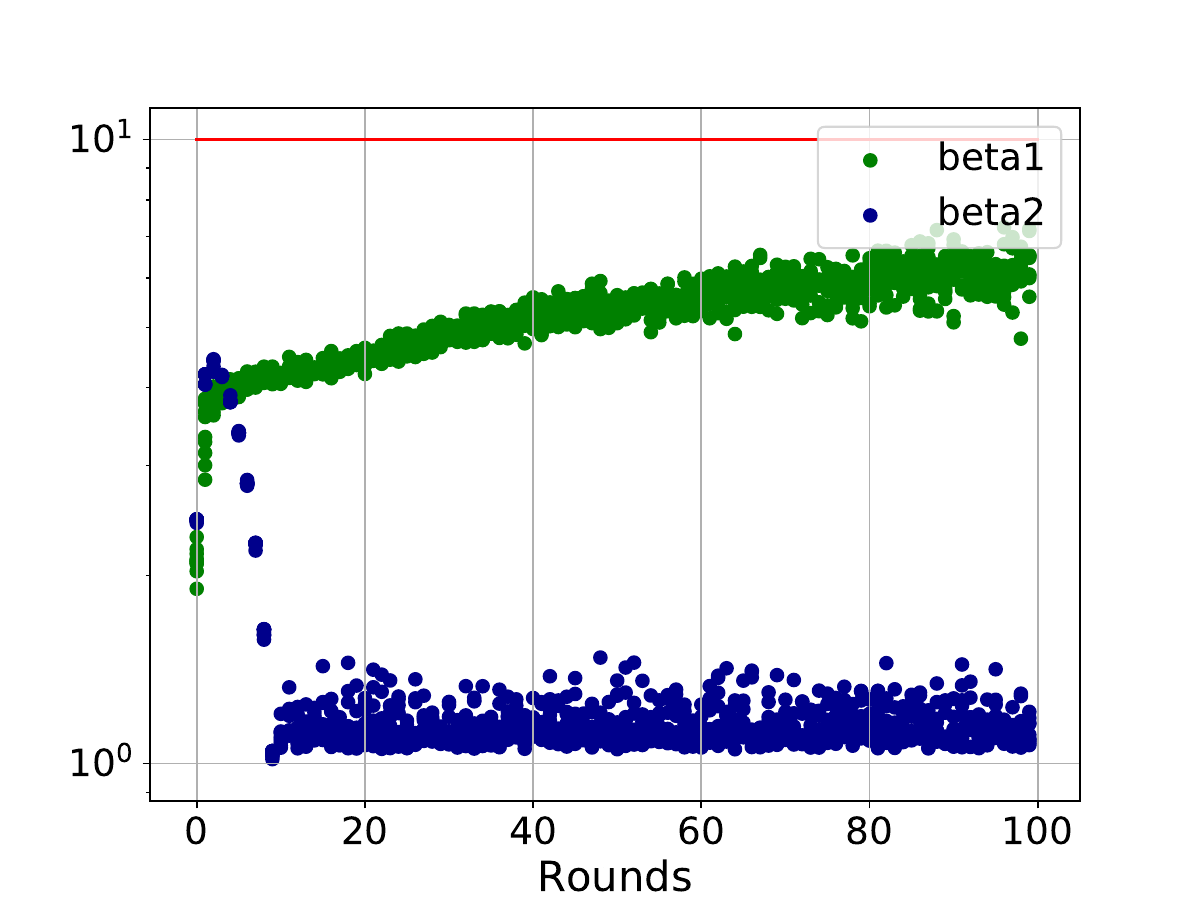}
    \caption{The empirical lower bounds for $\beta_1$ (green), $\beta_2$ (blue) for \textit{Pathological} distribution. The red line marks the applied compression rates (Top1\% on the left and Top10\% on the right).}
    \label{fig:compression_abl_path}
\end{figure}

\subsection{Experimental Details} \label{app:experimental details}
In this section we describe the specifics of our experimental setup.

\noindent \textbf{GLUE setup.}
The backbone is RoBERTa-base ($125$M parameters) with LoRA adapters \cite{hu2022lora} of rank $r = 4$ and scaling $\alpha = 2r = 8$, attached to the query, key and value projections and to every dense projection of each of the $12$ transformer blocks. 
Per round, each client performs at most $20$ local optimizer steps on
minibatches of $32$ sentences, i.e.\ it sees $640$ samples between two
communications. The local optimizer is Muon with momentum $0.95$ and $5$
Newton--Schulz iterations, applied to the matrix-shaped LoRA factors. Every run
lasts $30$ communication rounds.

\noindent\textbf{Hyperparameters.}
We present the hyperparameters of all methods for final runs. 

\noindent\textbf{CIFAR-10 experiments.} The following parameters are shared across all methods: batch size 256, 5 local steps.
The stepsizes were chosen for the methods with the grid search on a following grid: (0.00005, 0.0001, 0.0005, 0.001, 0.005, 0.01, 0.05, 0.1). Final stepsizes are listed in Table \ref{tab:stepsizes}. For spectral methods (\texttt{FedCALM}, \texttt{EF21-Muon}, \texttt{FedMuon}) \texttt{Muon} optimizer was implemented with 5 Newton-Schulz steps and momentum 0.95. For the rest of the methods we implemented \texttt{Adam} optimizer with momentums ($\beta_1, \beta_2 $)=(0.9, 0.999) and weight decay $\lambda$=0.001.

\begin{table}[ht]
\centering
\caption{Stepsizes for the final runs of \textit{CIFAR-10} experiments.}
\begin{tabular}{*5c}
\toprule
Method & Compression & \textit{Heter. distr.} & \textit{Path. distr.}  \\
\midrule
\multirow{2}{*}{\texttt{FedCALM}} & \textit{Top1\%} & 0.01 & 0.01\\
   &  \textit{Top10\%} & 0.01 & 0.01\\
\multirow{2}{*}{\texttt{SCAFCOM}} & \textit{Top1\%} & 0.0005 & 0.0005 \\
   &  \textit{Top10\%} & 0.0005 & 0.0005\\
\multirow{2}{*}{\texttt{EF21-Muon}} & \textit{Top1\%} & 0.0005 & 0.0001 \\
   &  \textit{Top10\%} & 0.0005 & 0.0001 \\
{\texttt{FedMuon}} & \textit{Top100\%} & 0.005 & 0.01 \\
{\texttt{LocalSGD}} & \textit{Top100\%} & 0.0001 & 0.0001 \\
\bottomrule
\end{tabular}\label{tab:stepsizes}
\end{table}

\noindent\textbf{GLUE experiments.} For every method the learning rate is
tuned over $0.0005,\, 0.001,\, 0.025, \, 0.005,\,
0.02$. Because Muon
normalizes its update direction, its step size does not live on the same scale
as the SGD or AdamW steps of the reference works, so the grid is treated as a
starting range rather than as a transferred setting. 
%The method-specific
%constant is tuned jointly with the learning rate rather than after it --
%$\theta \in \{0.01, 0.1\}$ for FedCALM, $\beta \in \{0.05, 0.2\}$ for SCAFCOM
%and EF21-Muon 
%optimal step size depends on them. 
The tuning is carried out at \textit{Top10\%} and \textit{Top1\%} indendently, so that a loss of accuracy
under aggressive sparsification cannot be an artifact of a step size inherited
from a milder compression level.

\begin{table}[ht]
\centering
\caption{Stepsizes for the final runs of \textit{SST-2} experiments.}
\begin{tabular}{*4c}
\toprule
Method & Compression & Stepsize  \\
\midrule
\multirow{2}{*}{\texttt{FedCALM}} & \textit{Top1\%} & 0.005\\
   &  \textit{Top10\%} & 0.005\\
\multirow{2}{*}{\texttt{SCAFCOM}} & \textit{Top1\%} & 0.0025 \\
   &  \textit{Top10\%} & 0.0025\\
\multirow{2}{*}{\texttt{EF21-Muon}} & \textit{Top1\%} & 0.0005 \\
   &  \textit{Top10\%} & 0.0005 \\
{\texttt{FedMuon}} & \textit{Top100\%} & 0.005 \\
{\texttt{LocalSGD}} & \textit{Top100\%} & 0.0025 \\
\bottomrule
\end{tabular}\label{tab:stepsizes_glue}
\end{table}

\noindent \textbf{Distributions.} 
For the \textit{GLUE} experiments, the training data is split among $M = 10$ clients, all of which participate in
every round. The split is IID: within each class the sentences are shuffled. Strong heterogeneity is not explored to ensure the stability of the training.

\noindent \textbf{CIFAR-10 experiments.}
For \textit{Pathological distribution} each client was assigned a number of objects from every class sampled uniformly from $U$[0.9,1.1]$\times$ 475.
For \textit{Heterogeneous distribution}, 4 clients were evenly assigned with objects from first 4 classes, and the rest 6 classes were evenly split among last 6 clients. However, each client has at least 2 objects from every class for stability of cross-enthropy loss. Tables \ref{tab:path_distr1}, \ref{tab:path_distr2}, \ref{tab:hetero_distr1}, \ref{tab:hetero_distr2} report exact amount of object of each class in clients datasets. 
\begin{table}[h]
\centering
\caption{Number of objects in client datasets for \textit{Pathological distibution} (clients 1-5).}
\begin{tabular}{c|ccccc}
\toprule
 & Client 1 & Client 2 & Client 3 & Client 4 & Client 5\\% & Client 6 & Client 7 & Client 8 & Client 9 & Client 10 \\
\midrule
Class 0 &1184 &1184 &1184 &1184 &2 \\%&2 &2 &2 &2 &2 \\
Class 1 &1185 &1185 &1185 &1185 &2 \\%&2 &2 &2 &2 &2 \\
Class 2 &1185 &1185 &1185 &1185 &2 \\%&2 &2 &2 &2 &2 \\
Class 3 &1184 &1184 &1184 &1184 &2 \\%&2 &2 &2 &2 &2 \\
Class 4 &2 &2 &2 &2 &790 \\%&790 &790 &790 &790 &790\\
Class 5 &2 &2 &2 &2 &791 \\%&791 \\%&791 &791 &791 &791 \\
Class 6 &2 &2 &2 &2 &791 \\%&791 \\%&791 &791 &791 &791 \\
Class 7 &2 &2 &2 &2 &790 \\%&790 \\%&790 &790 &790 &790 \\
Class 8 &2 &2 &2 &2 &790 \\%&790 \\%&790 &790 &790 &790 \\
Class 9 &2 &2 &2 &2 &790 \\%&790 \\%&790 &790 &790 &790 \\
\bottomrule
\end{tabular}\label{tab:path_distr1}
\end{table}

\begin{table}[h]
\centering
\caption{Number of objects in client datasets for \textit{Pathological distibution} (clients 6-10).}
\begin{tabular}{c|ccccc}
\toprule
 & Client 6 & Client 7 & Client 8 & Client 9 & Client 10 \\
\midrule
Class 0 &2 &2 &2 &2 &2 \\
Class 1 &2 &2 &2 &2 &2 \\
Class 2 &2 &2 &2 &2 &2 \\
Class 3 &2 &2 &2 &2 &2 \\
Class 4 &790 &790 &790 &790 &790\\
Class 5 &791&791 &791 &791 &791 \\
Class 6 &791&791 &791 &791 &791 \\
Class 7 &790&790 &790 &790 &790 \\
Class 8 &790&790 &790 &790 &790 \\
Class 9 &790&790 &790 &790 &790 \\
\bottomrule
\end{tabular}\label{tab:path_distr2}
\end{table}
\begin{table}[h]
\centering
\caption{Number of objects in client datasets for \textit{Heterogeneous distibution}.}
\begin{tabular}{c|ccccc}
\toprule
& Client 1 & Client 2 & Client 3 & Client 4 & Client 5 \\
\midrule
Class 0 &440 &448 &400 &428 &450 \\%&421 &448 &412 &430 &482 \\
Class 1 &501 &481 &488 &488 &502 \\%&510 &479 &523 &535 &503 \\
Class 2 &465 &479 &509 &518 &513 \\%&469 &496 &446 &504 &497 \\
Class 3 &470 &428 &476 &495 &437 \\%&447 &527 &490 &435 &486 \\
Class 4 &449 &482 &477 &486 &478 \\%&443 &506 &480 &460 &425 \\
Class 5 &495 &512 &547 &507 &524 \\%&497 &501 &500 &525 &482 \\
Class 6 &452 &428 &398 &438 &426 \\%&469 &450 &441 &419 &435 \\
Class 7 &471 &495 &534 &481 &519 \\%&516 &447 &514 &483 &477 \\
Class 8 &471 &468 &439 &439 &440 \\%&456 &437 &427 &452 &437 \\
Class 9 &536 &529 &482 &470 &461 \\%&522 &459 &517 &507 &526\\
\bottomrule
\end{tabular}\label{tab:hetero_distr1}
\end{table}
\begin{table}[h]
\centering
\caption{Number of objects in client datasets for \textit{Pathological distibution}.}
\begin{tabular}{c|ccccc}
\toprule
& Client 6 & Client 7 & Client 8 & Client 9 & Client 10 \\
\midrule
Class 0 &421 &448 &412 &430 &482 \\
Class 1 &510 &479 &523 &535 &503 \\
Class 2 &469 &496 &446 &504 &497 \\
Class 3 &447 &527 &490 &435 &486 \\
Class 4 &443 &506 &480 &460 &425 \\
Class 5 &497 &501 &500 &525 &482 \\
Class 6 &469 &450 &441 &419 &435 \\
Class 7 &516 &447 &514 &483 &477 \\
Class 8 &456 &437 &427 &452 &437 \\
Class 9 &522 &459 &517 &507 &526\\
\bottomrule
\end{tabular}\label{tab:hetero_distr2}
\end{table}

\newpage
\section{Proofs}\label{app:proofs}
First, we list some useful inequalities. Let $u,v$ be vectors in $\mathbb{R}^d$, $\{a_j\}_{j=1}^n$ be a sequence of vectors, and $\varphi : \mathbb{R}^d\rightarrow \mathbb{R}$ be a scalar $L_*$-smooth function.
\begin{center}
\begin{align}\displaystyle
    \label{Conj}\tag{Conj}\langle u,v\rangle &~\leq \|u\|\cdot \|v\|_* ,\\
    \label{Norm}\tag{Norm}\| u+v\| &~\leq \|u\|+ \|v\| ,\\
    \label{Jen}\tag{Jen}\|\frac{1}{n}\sum_{j=1}^na_j \|_* &~\leq\frac{1}{n}\sum_{j=1}^n \|a_j \|_*,\\
    \label{Smooth}\tag{Smooth} \varphi(v) &~\leq \varphi(u) + \langle \nabla \varphi(u),v-u\rangle + \frac{L}{2}\|u-v\|^2,\\
    \label{CS}\tag{CS}\left\|\sum_{i=1}^n a_i\right\|^2& ~\leq ~n\sum_{i=1}^n\|a_i\|^2.
\end{align}
\end{center}
Then, we prove a few auxiliary lemmas.
\begin{lemma}\label{lem:lmo}
    For any vector $u$, we have
    \begin{align*}
        \langle u,\operatorname{lmo}(u)\rangle= -\|u\|_*.
    \end{align*}
\end{lemma}
\begin{proof} The proof follows from the definition of conjugate norm.
    \begin{eqnarray*}
    -\|u\|_* = -\max_{\|v\|\leq 1}\langle u,v\rangle = \min_{\|v\|\leq 1}\langle u,v\rangle = \langle u,\arg\min_{\|v\|\leq 1}(\langle u,v\rangle)\rangle = \langle u,\operatorname{lmo}(u)\rangle.    
    \end{eqnarray*}
\end{proof}

Following \citet{kim2026convergence}, we introduce the error factor of Newton-Schultz steps $\varepsilon_q$. Formally, 
\begin{eqnarray}\label{eq:nss_error}
    \varepsilon_q = \max_{r=1\dots R}\max_{k=1\dots K}\max_{i=1\dots M}\|\operatorname{lmo}(p_i^{r,k}) - \texttt{NS}^q(p_i^{r,k})\|,
\end{eqnarray}
where $p_i^{r,k}$ denotes $\nabla f_i(x_i^{r,k}) - c_i^r + c^r $ in the context of the algorithm \ref{alg:fedcalm-d} and $p_i^{r,k}=m_i^{r,k} - c_i^r + c^r $ for the algorithm \ref{alg:fedcalm}. \citet{kim2026convergence} showed, that $\varepsilon_q<1$ and approaches zero at a doubly exponential rate over $q$.
\begin{corollary}\label{eq:nss_bound} During the run of Algorithm \ref{alg:fedcalm-d} or \ref{alg:fedcalm}, the following holds for any $p_i^{r,k}$ (as defined earlier):
\begin{align*}
    \|\texttt{NS}(p_i^{r,k})\|\leq 1 + \varepsilon_q.
\end{align*}
\begin{proof}
We utilize the unit length of LMO and introduced error of Newton-Schultz steps:
    \begin{eqnarray*}
        \|\texttt{NS}^q(p_i^{r,k})\|\overset{\ref{Norm}}{\leq} \|\operatorname{lmo}(p_i^{r,k})\| + \|\operatorname{lmo}(p_i^{r,k}) - \texttt{NS}^q(p_i^{r,k})\|
        \overset{(\ref{eq:nss_error})}{\leq} \|\operatorname{lmo}(p_i^{r,k})\| + \varepsilon_q
        {\leq} 1 + \varepsilon_q.
    \end{eqnarray*}
\end{proof}
    
\end{corollary}

\begin{lemma}\label{lem:dist}
    During the run of Algorithm \ref{alg:fedcalm-d} or \ref{alg:fedcalm}, the following inequality holds for each $i=1,\dots M$ and $r=1,\dots R+1$: 
    \begin{align*}
        \|x_i^{r,k} - x_r\| \leq \eta k (1+\varepsilon_q).
    \end{align*}
\end{lemma}
\begin{proof}
    The proof relies on the definition of LMO. Since the $\arg\min$ is taken over the $\|\cdot\|$-norm ball, the LMO output has the norm of $1$ or less. Thus,
    \begin{eqnarray*}
        \|x_i^{r,k} - x^{r}\| = \|x_i^{r,0} - x_i^{r,0}\|
        = \|\sum_{l=0}^{k-1}(x_i^{r,l} - x_i^{r,l+1})\|
        \overset{\ref{Norm}}{\leq} \sum_{l=0}^{k-1}\|x_i^{r,l} - x_i^{r,l+1}\|
        %\\&\leq & \eta\sum_{k=0}^K1
        %\\&=& \eta K
        .
    \end{eqnarray*}
    As we noted, each update $x^{r,l}_i-x_i^{r,l+1}$ is an output of \texttt{NS} steps scaled by $\eta$. Corollary \ref{eq:nss_bound} implies that the norm of every update equals to or less than $\eta(1+\varepsilon_q)$, which concludes the proof:
        \begin{eqnarray*}
        \|x_i^{r,k} - x^{r}\| &{\leq}& \sum_{l=0}^{k-1}\|x_i^{r,l} - x_i^{r,l+1}\|
        \leq \sum_{l=0}^k\eta(1+\varepsilon_q)
        =\eta k(1+\varepsilon_q)
        .
    \end{eqnarray*}
\end{proof}
\begin{corollary}\label{cor:y}
    From the definition of $y_i^r$ it is easy to observe that
    $\|y_i^r\|\leq \eta K(1+\varepsilon_q)$ for any client $i$ and round $r$.
\end{corollary}

\begin{lemma}\label{lem:e}
    During the run of Algorithm \ref{alg:fedcalm-d} or \ref{alg:fedcalm}, the following inequality holds for each $i=1,\dots M$ and $r=1,\dots R+1$: 
    \begin{align*}
        \|e_i^r\| \leq  \eta K \beta(1+\varepsilon_q).
    \end{align*}
\end{lemma}
\begin{proof}
    This proof resembles classic error analysis for \texttt{EF}.
    \begin{eqnarray*}
        \|e_i^r\| &=& \|e_i^{r-1} + y_i^r - v_i^r\|
        \\&=& \|e_i^{r-1} + y_i^r - C(y_i^r + e_i^{r-1})\|
        \\&\overset{(Def.\ref{eq:compres})}{\leq}& (1-\frac1\beta)\|y_i^r+e_i^{r-1}\| 
        \\&\overset{\ref{Norm}}{\leq}& (1-\frac1\beta)\|y_i^r\|+(1-\frac1\beta)\|e_i^{r-1}\|.
    \end{eqnarray*}
    Now, we unroll a recurcion over $\|e_i^{r}\|$ to $r=0$:
    \begin{eqnarray*}
        \|e_i^r\| &{\leq}& (1-\frac1\beta)\|y_i^r\|+(1-\frac1\beta)\|e_i^{r-1}\|
        \\&\dots&
        \\&\leq& \sum_{t=0}^r (1-\frac1\beta)^{r+1-t}\|y_i^{t}\| + (1-\frac1\beta)^r\|e_i^0\|
        \\&\overset{(Cor.\ref{cor:y})}{\leq}& \sum_{t=0}^r (1-\frac1\beta)^{r+1-t}\eta K (1+\varepsilon_q)+ (1-\frac1\beta)^r\|e_i^0\|
        \\&{\leq}& \eta K(1+\varepsilon_q) \sum_{t=0}^\infty (1-\frac1\beta)^{t}+ (1-\frac1\beta)^r\|e_i^0\|
        \\&=& \eta K \beta (1+\varepsilon_q) + (1-\frac1\beta)^r\|e_i^0\|.
    \end{eqnarray*}
    Finally, we note that $\|e_i^0\|=0$ due to initialization.
    
\end{proof}

\begin{lemma}\label{lem:step}
    During the run of Algorithm \ref{alg:fedcalm-d} or \ref{alg:fedcalm}, the following inequality holds for each round $r=1,\dots R+1$: 
    \begin{align*}
        \|x^{r} - x^{r-1}\| \leq  6\eta K \beta.
    \end{align*}
\end{lemma}
\begin{proof}
    We start with the definition of the step:
    \begin{eqnarray*}
        \|x^{r} - x^{r-1}\| &=& \|\frac1M \sum_{i=1}^M v_i^r\|
        \\&\overset{(\ref{Norm})}{\leq}& \frac1M \sum_{i=1}^M\|C(y_i^r + e_i^{r-1})\|
        \\&\overset{(\ref{Norm})}{\leq}& \frac1M \sum_{i=1}^M\|C(y_i^r + e_i^{r-1}) - y_i^r - e_i^{r-1}\| + \frac1M \sum_{i=1}^M\|y_i^r + e_i^{r-1}\|
        \\&\overset{(Def.\ref{eq:compres})}{\leq}& \frac1M \sum_{i=1}^M(1-\frac1\beta)\|y_i^r + e_i^{r-1}\| + \frac1M \sum_{i=1}^M\|y_i^r + e_i^{r-1}\|
        \\&=& \frac{2-\frac1\beta}M \sum_{i=1}^M\|y_i^r + e_i^{r-1}\|
        \\&=& \frac{2-\frac1\beta}M \sum_{i=1}^M(\|y_i^r \|+\| e_i^{r-1}\|)
        .
    \end{eqnarray*}
    Then, we apply Lemma \ref{lem:e} and Corollary \ref{cor:y} to bound $\|e_i^{r-1}\|$ and $\|y_i^r\|$ individually:
    \begin{eqnarray*}
        \|x^{r} - x^{r-1}\| &\leq& \frac{2-\frac1\beta}M \sum_{i=1}^M(\|y_i^r \|+\| e_i^{r-1}\|)
        \\&\leq& \frac{2-\frac1\beta}M \sum_{i=1}^M(\eta K(1+\varepsilon_q)+\eta K \beta(1+\varepsilon_q))
        \\&=& \frac{(2-\frac1\beta)(\beta+1)}M (1+\varepsilon_q)\eta K\sum_{i=1}^M 1
        \\&=& {(1+2\beta - \frac1\beta)}\eta K(1+\varepsilon_q)
        \\&\leq& {6\beta}\eta K
        ,
    \end{eqnarray*}
    where the last inequality was achived via trivial bounds $1\leq \beta,\frac1\beta > 0$ and $\varepsilon_q\leq1$.
     
\end{proof}

\subsection{Proof of Theorem \ref{th:conv_fedcalm-d}}
Now, we are ready to prove the convergence for Algorithm \ref{alg:fedcalm-d}.
\begin{theorem}[Theorem \ref{th:conv_fedcalm-d}]
    Suppose Assumption \ref{ass:smooth} holds. Then Algorithm \ref{alg:fedcalm-d} to achieve $\varepsilon$-accuracy needs
    \begin{equation*}
     R = \mathcal{O}\left(\frac{\beta^2\Delta(L_*+L_2)}{\varepsilon^2(1-\varepsilon_q)^2}\right) \text{~~communication rounds,}
    \end{equation*}
    where $\displaystyle\varepsilon = \mathbb{E}\left[\sum_{r=0}^{R-1}\|\nabla f(x^r)\|_*\right]$, and $\Delta=f(x^0)-{f(x^R)}$, and $\varepsilon_q$ is the error of Newton-Schultz iterations.
\end{theorem}
\newcommand{\tx}{\tilde{x}}
\begin{proof}
    Classicly for the \texttt{EF}, we start with defining a virtual sequence.
    \begin{align*}
        \tx^r \overset{\text{def}}{\equiv} x^r - \frac{1}{M}\sum_{i}e^r_i.
    \end{align*}
    The vectors $\tx^r$ are useful for the proof but never explicitly calculated in the algorithm (hence the name \textit{virtual}).
    The main feature of the sequence is its update rule:
    \begin{align*}
    \tx^{r+1} =& x^{r+1} - \frac{1}{M}\sum_{i}e^{r+1}_i
    \\ =&  (x^r - \frac{1}{M} \sum_{i}v^{r+1}_i)
    - \frac{1}{M}\sum_{i} (e^{r}_i + y^{r+1}_i - v_i^{r+1})
    \\=& x^r -\frac{1}{M} \sum_{i}\left(v^{r+1}_i +e^r_i+y^{r+1}_i-v^{r+1}_i\right)
    \\=& (x^r - \frac{1}{M} \sum_{i}e^{r+1}_i)- \frac{1}{M} \sum_{i}y^{r+1}_i
    \\=& \tx^r -  \frac{1}{M} \sum_{i}y^{r+1}_i.
    \end{align*}
    Now, we write the Smoothness inequality substituting $v=\tx^{r+1}$ and $u=\tx^r$:
    \begin{eqnarray*}
        f(\tx^{r+1}) &\overset{(\ref{Smooth})}{\leq} & f(\tx^r) - \langle \nabla f(\tx^r),\frac{1}{M}\sum_{i=1}^M y^{r+1}_i\rangle + \frac{L_*}{2}\|\frac1M \sum_{i=1}^M y_i^{r+1}\|^2
        \\&\overset{(\ref{CS})}{\leq}& f(\tx^r) - \frac{1}{M}\sum_{i=1}^M\langle \nabla f(\tx^r), y^{r+1}_i\rangle + \frac{L_*}{2M}\sum_{i=1}^M\| y_i^{r+1}\|^2.
    \end{eqnarray*}
    Let us analyse scalar product separately.
    \begin{eqnarray*}
        -\langle \nabla f(\tx^r), y^{r+1}_i\rangle &=& -\langle \nabla f(\tx^r), x^r-x^{r,K}_i\rangle
        \\&=& \langle \nabla f(\tx^r), \sum_{k=0}^{K-1}\eta \texttt{NS}(\nabla f_i(x_i^{r,k}) - c_i^r + c^r) \rangle
        \\&=& \langle \nabla f(\tx^r), \sum_{k=0}^{K-1}\eta \operatorname{lmo}(\nabla f_i(x_i^{r,k}) - c_i^r + c^r) \rangle
        \\&&+ \langle \nabla f(\tx^r), \sum_{k=0}^{K-1}\eta [\operatorname{lmo}-\texttt{NS}^q](\nabla f_i(x_i^{r,k}) - c_i^r + c^r) \rangle
        \\&=& \eta\sum_{k=0}^{K-1}\langle \nabla f(\tx^r),  \operatorname{lmo}(\nabla f_i(x_i^{r,k}) - c_i^r + c^r) \rangle
        \\&&+ \eta\sum_{k=0}^{K-1}\langle \nabla f(\tx^r),  [\operatorname{lmo}-\texttt{NS}^q](\nabla f_i(x_i^{r,k}) - c_i^r + c^r) \rangle
        \\&\overset{(\ref{Conj})}{\leq}& \eta\sum_{k=0}^{K-1}\langle \nabla f(\tx^r),  \operatorname{lmo}(\nabla f_i(x_i^{r,k}) - c_i^r + c^r) \rangle
        \\&&+ \eta\sum_{k=0}^{K-1}\| \nabla f(\tx^r)\|_*  \|[\operatorname{lmo}-\texttt{NS}^q](\nabla f_i(x_i^{r,k}) - c_i^r + c^r) \|
        \\&\overset{(\ref{eq:nss_error})}{\leq}& \eta\sum_{k=0}^{K-1}\langle \nabla f(\tx^r),  \operatorname{lmo}(\nabla f_i(x_i^{r,k}) - c_i^r + c^r) \rangle
        \\&&+ \eta\sum_{k=0}^{K-1}\| \nabla f(\tx^r)\|_*  \varepsilon_q
        % \\&=& \eta\sum_{k=0}^{K-1}\langle \nabla f_i(x_i^{r,k}) - c_i^r+c^r,  \operatorname{lmo}(\nabla f_i(x_i^{r,k}) - c_i^r + c^r) \rangle
        % \\&&+ \eta\sum_{k=0}^{K-1}\langle \nabla f(\tx^r)-\nabla f_i(x_i^{r,k}) + c_i^r-c^r,  \operatorname{lmo}(\nabla f_i(x_i^{r,k}) - c_i^r + c^r) \rangle
        .
    \end{eqnarray*}
    Here, we transitioned from inexact \texttt{NS} to LMO, gaining the error term proportonal to $\varepsilon_q$. Next, we operate over the first term to utilize Lemma \ref{lem:lmo}.
    \begin{eqnarray*}
        -\langle \nabla f(\tx^r), y^{r+1}_i\rangle &{\leq}& \eta\sum_{k=0}^{K-1}\langle \nabla f(\tx^r),  \operatorname{lmo}(\nabla f_i(x_i^{r,k}) - c_i^r + c^r) \rangle
        \\&&+ \eta\sum_{k=0}^{K-1}\| \nabla f(\tx^r)\|_*  \varepsilon_q
        \\&=& \eta\sum_{k=0}^{K-1}\langle \nabla f_i(x_i^{r,k}) - c_i^r+c^r,  \operatorname{lmo}(\nabla f_i(x_i^{r,k}) - c_i^r + c^r) \rangle
        \\&&+ \eta\sum_{k=0}^{K-1}\langle \nabla f(\tx^r)-\nabla f_i(x_i^{r,k}) + c_i^r-c^r,  \operatorname{lmo}(\nabla f_i(x_i^{r,k}) - c_i^r + c^r) \rangle
        \\&&+ \eta\sum_{k=0}^{K-1}\| \nabla f(\tx^r)\|_*  \varepsilon_q
        .
    \end{eqnarray*}
    
    We apply Lemma \ref{lem:lmo} to the first term, retrieving the conjugate norm. For the second term, we use \ref{Conj} inequality keeping in mind the unit norm of LMO.
    \begin{eqnarray*}
        -\langle \nabla f(\tx^r), y^{r+1}_i\rangle &\leq& -\eta\sum_{k=0}^{K-1}\| \nabla f_i(x_i^{r,k}) - c_i^r+c^r\|_*
        \\&&+ \eta\sum_{k=0}^{K-1}\| \nabla f(\tx^r)-\nabla f_i(x_i^{r,k}) + c_i^r-c^r\|_*\cdot 1
        \\&&+ \eta\sum_{k=0}^{K-1}\| \nabla f(\tx^r)\|_*  \varepsilon_q
        \\&\overset{\ref{Norm}}{\leq}& -\eta\sum_{k=0}^{K-1}\| \nabla f(x^{r})\|_*
        \\&&+ 2\eta\sum_{k=0}^{K-1}\| \nabla f(x^r)-\nabla f_i(x_i^{r,k}) + c_i^r-c^r\|_*
        \\&&+\eta\sum_{k=0}^{K-1}\| \nabla f(\tx^r)-\nabla f(x^r)\|_*
        \\&&+ \eta\sum_{k=0}^{K-1}(\| \nabla f(\tx^r) - \nabla f(x^r)\|_* + \|\nabla f(x^r)\|_*)  \varepsilon_q
        \\&\overset{\ref{Norm}}{\leq}& -\eta(1-\varepsilon_q)\sum_{k=0}^{K-1}\| \nabla f(x^{r})\|_*
        \\&&+ 2\eta\sum_{k=0}^{K-1}\| \nabla f(x^r)-\nabla f_i(x_i^{r,k}) + c_i^r-c^r\|_*
        \\&&+\eta(1+\varepsilon_q)\sum_{k=0}^{K-1}\| \nabla f(\tx^r)-\nabla f(x^r)\|_*
        .
    \end{eqnarray*}
    The last term can be easilty bounded via Lemma \ref{lem:e}:
    \begin{eqnarray*}
        \| \nabla f(\tx^r)-\nabla f(x^r)\|_* &\overset{(Ass.\ref{ass:smooth})}{\leq}& L_*\|\tx^r-x^r\|
        \\&=& L_*\|\frac1M\sum_{i=1}^Me_i^r\|
        \\&\overset{(\ref{Norm})}{\leq}& L_*\frac1M\sum_{i=1}^M\|e_i^r\|
        \\&\overset{(Lem.\ref{lem:e})}{\leq}& \frac{L_*}{M}\sum_{i=1}^M\eta K \beta(1+\varepsilon_q)
        \\&=& {L_*}\eta K \beta(1+\varepsilon_q).
    \end{eqnarray*}
    Now, the second term deserves much attention. First, we remove the dependance on $k$:
    \begin{eqnarray*}
        \| \nabla f(x^r)-\nabla f_i(x_i^{r,k}) + c_i^r-c^r\|_* &\overset{\ref{Norm}}{\leq}& \| \nabla f(x^r)-c^r\|_* + \|\nabla f_i(x^r) - c_i^r\|_* + \| \nabla f_i(x^r)-\nabla f_i(x_i^{r,k})\|_* 
        \\&\overset{\ref{ass:smooth}}{\leq}& \| \nabla f(x^r)-c^r\|_* + \|\nabla f_i(x^r) - c_i^r\|_* + L_*\|x^r-x_i^{r,k}\| 
        \\&\overset{(Lem.\ref{lem:dist})}{\leq}& \| \nabla f(x^r)-c^r\|_* + \|\nabla f_i(x^r) - c_i^r\|_* + L_*\eta k(1+\varepsilon_q)
        \\&\overset{}{\leq}& \| \nabla f(x^r)-c^r\|_* + \|\nabla f_i(x^r) - c_i^r\|_* + L_*\eta K(1+\varepsilon_q).
    \end{eqnarray*}
    Next, we bound the difference of control variate and respective gradient.
    \begin{eqnarray*}
        \|\nabla f_i(x^r) - c_i^r\|_* &{=} & \|\nabla f_i(x^r) - c_i^{r-1} - \delta_i^r\|_*
        \\&= & \|\nabla f_i(x^r) - c_i^{r-1} - C(\nabla f_i(x_i^{r-1,K}) - c_i^{r-1})\|_*
        \\&\overset{(\ref{Norm})}{\leq} & \|\nabla f_i(x_i^{r-1,K}) - c_i^{r-1} - C(\nabla f_i(x_i^{r-1,K}) - c_i^{r-1})\|_* 
        \\&&+ \|\nabla f_i(x_i^{r-1,K}) - \nabla f_i(x^r)\|_*
        \\&\overset{(Def.\ref{eq:compres})}{\leq} & (1-\frac1\beta) \|\nabla f_i(x_i^{r-1,K}) - c_i^{r-1}\|_* + \|\nabla f_i(x_i^{r-1,K}) - \nabla f_i(x^r)\|_*
        \\&\overset{(\ref{Norm})}{\leq} & (1-\frac1\beta) \|\nabla f_i(x^{r-1}) - c_i^{r-1}\|_* + 2 \|\nabla f_i(x_i^{r-1,K}) - \nabla f_i(x^{r-1})\|_*
        \\&&+ \|\nabla f_i(x_i^{r-1}) - \nabla f_i(x^r)\|_*
        \\&\overset{(Ass.\ref{ass:smooth})}{\leq} & (1-\frac1\beta) \|\nabla f_i(x^{r-1}) - c_i^{r-1}\|_* + 2L_* \|x^{r-1,K} - x^{r-1}\| 
        \\&&+ L_* \|x^{r-1} - x^{r}\|
        \\&\overset{(Lem.\ref{lem:step})}{\leq} & (1-\frac1\beta) \|\nabla f_i(x^{r-1}) - c_i^{r-1}\|_* + 18L_* \eta K \beta.
    \end{eqnarray*}
    Then, we unroll a recursion over $r$ to the first round.
    \begin{eqnarray*}
        \|\nabla f_i(x^r) - c_i^r\|_* &{\leq} & (1-\frac1\beta) \|\nabla f_i(x^{r-1}) - c_i^{r-1}\|_* + 18 L_* \eta K \beta
        \\&\dots&
        \\&\leq & (1-\frac1\beta)^r \|\nabla f_i(x^{0}) - c_i^{0}\|_* + \sum_{t=0}^{r-1}18(1-\frac1\beta)^t L_* \eta K \beta
        \\&\leq & (1-\frac1\beta)^r \|\nabla f_i(x^{0}) - c_i^{0}\|_* + 18L_*\eta K \beta\sum_{t=0}^{\infty}(1-\frac1\beta)^t
        \\&= & (1-\frac1\beta)^r \|\nabla f_i(x^{0}) - c_i^{0}\|_* + 18 L_*\eta K \beta^2 
        .
    \end{eqnarray*}
    Again, due to initialization the first term is zero, resulting in the following bound:
    \begin{eqnarray}\label{eq:c_grad}
        \|\nabla f_i(x^r) - c_i^r\|_* &{\leq} & 18L_*\eta K \beta^2
        .
    \end{eqnarray}
    This gives an easy estimate for the average control varible:
    \begin{eqnarray*}
        \|\nabla f(x^r) - c^r\|_* = \|\frac1M\sum_{i=1}^M(\nabla f_i(x^r) - c_i^r)\|_*
        \overset{(\ref{Norm})}{\leq} \frac1M\sum_{i=1}^M\|\nabla f_i(x^r) - c_i^r\|_*
        .
    \end{eqnarray*}
    Now we plug in \eqref{eq:c_grad} and gain
        \begin{eqnarray*}
        \|\nabla f(x^r) - c^r\|_* \leq \frac1M\sum_{i=1}^M\|\nabla f_i(x^r) - c_i^r\|_*
        \leq  \frac1M\sum_{i=1}^M18L_*\eta K\beta^2
        =  18L_*\eta K\beta^2
        .
    \end{eqnarray*}
    This allows us to continue with the scalar product analysis. We restate the descent inequality for convinience:
    \begin{eqnarray*}
        -\langle \nabla f(\tx^r), y^{r+1}_i\rangle &\leq&  -\eta(1-\varepsilon_q)\sum_{k=0}^{K-1}\| \nabla f(x^{r})\|_*
        \\&&+ 2\eta\sum_{k=0}^{K-1}\| \nabla f(x^r)-\nabla f_i(x_i^{r,k}) + c_i^r-c^r\|_*
        \\&&+\eta(1+\varepsilon_q)\sum_{k=0}^{K-1}\| \nabla f(\tx^r)-\nabla f(x^r)\|_*
        .
    \end{eqnarray*}
    All acquired bounds bring us the following:
    \begin{eqnarray*}
        -\langle \nabla f(\tx^r), y^{r+1}_i\rangle &\leq&  -\eta(1-\varepsilon_q)K\| \nabla f(x^{r})\|_*
        \\&&+ 2\eta\sum_{k=0}^{K-1}\big(\|\nabla f(x^r) - c^r\|_* + \|\nabla f_i(x^{r}) - c_i^r\|_*
        \\&&+ \|\nabla f_i(x_i^{r,k}) - \nabla f_i(x^r)\|_*\big)
        \\&&+\eta(1+\varepsilon_q)\sum_{k=0}^{K-1}\| \nabla f(\tx^r)-\nabla f(x^r)\|_*
        \\&\overset{(\ref{eq:c_grad})}{\leq}&  -\eta(1-\varepsilon_q)K\| \nabla f(x^{r})\|_*
        \\&&+ 2\eta\sum_{k=0}^{K-1}\big(36L_*\eta K\beta^2 + \|\nabla f_i(x_i^{r,k}) - \nabla f_i(x^r)\|_*\big)
        \\&&+\eta(1+\varepsilon_q)\sum_{k=0}^{K-1}\| \nabla f(\tx^r)-\nabla f(x^r)\|_*
        \\&\overset{(Ass.\ref{ass:smooth})}{\leq}&  -\eta(1-\varepsilon_q)K\| \nabla f(x^{r})\|_*
        \\&&+ 2\eta\sum_{k=0}^{K-1}\big(36L_*\eta K\beta^2 + \|x_i^{r,k} - x^r\|\big)
        \\&&+\eta(1+\varepsilon_q)\sum_{k=0}^{K-1}\|\tx^r-x^r\|
        \\&\overset{(Lem.\ref{lem:dist},\ref{lem:e})}{\leq}&  -\eta(1-\varepsilon_q)K\| \nabla f(x^{r})\|_*
        \\&&+ 2\eta\sum_{k=0}^{K-1}\big(36L_*\eta K\beta^2 + L_*\eta K(1+\varepsilon_q)\big)
        \\&&+\eta(1+\varepsilon_q)\sum_{k=0}^{K-1}\eta K \beta (1+\varepsilon_q)
        \\&\leq&  -\eta(1-\varepsilon_q)K\| \nabla f(x^{r})\|_*
        + 76L_*\eta^2 K^2\beta^2
        +4\eta^2 K^2L_*\beta
        \\&\leq&  -\eta(1-\varepsilon_q)K\| \nabla f(x^{r})\|_*
        + 80L_*\eta^2 K^2\beta^2
        ,
    \end{eqnarray*}
    where we trivially estimate $1+\varepsilon_q\leq 2\leq 2\beta$.

    Finally, we return to the descent equation.
    \begin{eqnarray*}
        f(\tx^{r+1}) &{\leq}& f(\tx^r) - \frac{1}{M}\sum_{i=1}^M\langle \nabla f(\tx^r), y^{r+1}_i\rangle + \frac{L_*}{2M}\sum_{i=1}^M\| y_i^{r+1}\|^2.
    \end{eqnarray*}
    For $\|y_i^{r+1}\|^2$, we apply Corollary \ref{cor:y}.
    \begin{eqnarray*}
        \frac{L_*}{2M}\sum_{i=1}^M\| y_i^{r+1}\|^2
        \overset{(Cor.\ref{cor:y})}{\leq} \frac{L_*}{2M}\sum_{i=1}^M(\eta K(1+\varepsilon_q))^2
        = 2{L_*\eta^2K^2}.
    \end{eqnarray*}
    This plug in the obtained result and take the expectation.
    \begin{eqnarray*}
        \mathbb{E}f(\tx^{r+1}) &{\leq}& \mathbb{E}f(\tx^r) -\eta(1-\varepsilon_q)K\mathbb{E}\| \nabla f(x^{r})\|_*
        + 80L_*\eta^2 K^2\beta^2 + 2{L_*\eta^2K^2}
        .
    \end{eqnarray*}
    We rearrange the terms and take the average over $r=0,\dots R-1:$
    \begin{eqnarray*}
        \mathbb{E}\left[\frac1R\sum_{r=0}^{R-1}\eta(1-\varepsilon_q)K\| \nabla f(x^{r})\|_* \right]&{\leq}& \frac{1}{R}\sum_{r=0}^{R-1}\mathbb{E}\left[f(\tx^r) -f(\tx^{r+1})\right]
        + 82L_*\eta^2 K^2\beta^2
        \\&\leq& \frac{\Delta}{R}
        + 82L_*\beta^2\eta^2 K^2
        ,
    \end{eqnarray*}
    where $f(x^0) -f(x^R)\leq f(x^0)-f(x^*)=\Delta$. Now, we divide the equation by $\eta K (1-\varepsilon_q)$.
    \begin{eqnarray*}
        \mathbb{E}\left[\frac1R\sum_{r=0}^{R-1}\| \nabla f(x^{r})\|_*\right] &{\leq}& \frac{\Delta}{\eta(1-\varepsilon_q)K R}
        + \frac{34L_*\beta^2\eta K}{1-\varepsilon_q}
        ,
    \end{eqnarray*}
    Now we choose the stepsize $\displaystyle\eta = \Theta \left(\frac{\sqrt{\Delta}}{K\beta\sqrt{RL_*}}\right) $ and acquire:
    \begin{eqnarray*}
        \mathbb{E}\left[\frac1R\sum_{r=0}^{R-1}\| \nabla f(x^{r})\|_*\right] = \mathcal{O}\left(\frac{\beta\sqrt{\Delta L_*}}{(1-\varepsilon_q) \sqrt{R}}
        \right).
    \end{eqnarray*}
    Lastly, we denote $\varepsilon = \mathbb{E}\left[\frac1R\sum_{r=0}^{R-1}\| \nabla f(x^{r})\|_*\right]$ and rearrange the terms to obtain round complexity.
    \begin{eqnarray*}
        R = \mathcal{O}\left(\frac{\beta^2{\Delta L_*}}{\varepsilon^2(1-\varepsilon_q)^2}
        \right).
    \end{eqnarray*}
\end{proof}

\subsection{Proof of Theorem \ref{th:conv_fedcalm}}
\begin{remark}
    We recall that Lemmas \ref{lem:lmo}, \ref{lem:dist}, \ref{lem:e}, \ref{lem:step} all are applicable to Algorithm \ref{alg:fedcalm}, as well as Corollaries \ref{eq:nss_bound}, \ref{cor:y}. 
\end{remark}
\begin{theorem}
    Suppose Assumptions \ref{ass:smooth}, \ref{ass:non-convex}, and \ref{ass:stoch} hold. Then Algorithm \ref{alg:fedcalm} to achieve $\varepsilon$-accuracy needs
    \begin{equation*}
     R = \mathcal{O}\left(\frac{{\Delta L_* \beta^2}}{\varepsilon^2(1-\varepsilon_q)^2} + \frac{\Delta L_*\rho^2\sigma^2\beta}{M\varepsilon^4(1-\varepsilon_q)^4}\right)
    \end{equation*}
    communication rounds, where $\displaystyle\varepsilon = \mathbb{E}\left[\sum_{r=0}^{R-1}\|\nabla f(x^r)\|_*\right]$, and $\Delta=f(x^0)-{f(x^R)}$, $\tau$ is a unit of time required to send one uncompressed model to the server, $\varepsilon_q$ is the error of Newton-Schulz iterations.
\end{theorem}
We divide the lengthy proof into lemmas for clarity. First one deals with the difference of the gradient and momentum:
\begin{lemma}\label{lem:m_f}
    Suppose Assumptions \ref{ass:smooth}, \ref{ass:non-convex}, and \ref{ass:stoch} hold. Then during the run of Algorithm \ref{alg:fedcalm} the following holds:
    \begin{eqnarray*}
        \mathbb{E}\|m^{r,K} - \nabla f(x^r)\|_* \leq {\rho \sigma}{\sqrt\alpha} + 2 L_*\eta K +\frac{6 L_*\eta K \beta}{1-(1-\alpha)^K},
    \end{eqnarray*}
    where $m^{r,k} = \frac1M\sum_{i=1}^Mm_i^{r,k} $.
\end{lemma}
\begin{proof}
    First, we bound the discrepancy of momentum and gradient throughout the round. We denote $g_i^{r,k}=g_i(x_i^{r,k},\xi_i^{r,k}), ~ g^{r,k} = \frac1M\sum_{i=1}^M g_i^{r,k}$ and $\nabla f_i^{r,k}=\nabla f_i(x_i^{r,k}), ~\nabla f^{r,k} = \frac1M\sum_{i=1}^M\nabla f_i^{r,k} $ for the clarity of notation.
    \begin{eqnarray*}
        m^{r,K} - \nabla f(x^r) &=& (1-\alpha)m^{r,K-1} + \alpha g^{r,K-1} - \nabla f(x^r)
        \\&=& (1-\alpha)(m^{r,K-1} - \nabla f(x^r)) + \alpha (g^{r,K-1} - \nabla f^{r,K-1})
        \\&&+ \alpha (\nabla f^{r,K-1} - \nabla f(x^r))
        \\&=&\dots
        \\&=& (1-\alpha)^K(m^{r,0} - \nabla f(x^r)) + \alpha \sum_{l=0}^{K-1} (1-\alpha)^{K-1-l} (g^{r,l} - \nabla f^{r,l})
        \\&&+ \alpha \sum_{l=0}^{K-1} (1-\alpha)^{K-1-l}(\nabla f^{r,l} - \nabla f(x^r))
        .
    \end{eqnarray*}
    Next, we continue the equality to the previous round.
    \begin{eqnarray*}
        m^{r,K} - \nabla f(x^r) &=& (1-\alpha)^K(m^{r,0} - \nabla f(x^r)) + \alpha \sum_{l=0}^{K-1} (1-\alpha)^{l} (g^{r,K-1-l} - \nabla f^{r,K-1-l})
        \\&&+ \alpha \sum_{l=0}^{K-1} (1-\alpha)^{l}(\nabla f^{r,K-1-l} - \nabla f(x^r))
        \\&=& (1-\alpha)^K(m^{r-1,K} - \nabla f(x^{r-1})) + (1-\alpha)^K(\nabla f(x^{r-1} - \nabla f(x^r))
        \\&&+ \alpha \sum_{l=0}^{K-1} (1-\alpha)^{l} (g^{r,K-1-l} - \nabla f^{r,K-1-l})
        \\&&+ \alpha \sum_{l=0}^{K-1} (1-\alpha)^{l}(\nabla f^{r,K-1-l} - \nabla f(x^r))
        .
    \end{eqnarray*}
    Now, we unroll the recursion to the $r=0$:
    \begin{eqnarray*}
        m^{r,K} - \nabla f(x^r) &=& (1-\alpha)^K(m^{r-1,K} - \nabla f(x^{r-1})) + (1-\alpha)^K(\nabla f(x^{r-1}) - \nabla f(x^r))
        \\&&+ \alpha \sum_{l=0}^{K-1} (1-\alpha)^{l} (g^{r,K-1-l} - \nabla f^{r,K-1-l})
        \\&&+ \alpha \sum_{l=0}^{K-1} (1-\alpha)^{l}(\nabla f^{r,K-1-l} - \nabla f(x^r))
        \\&=&\dots
        \\&=& (1-\alpha)^{(r-1)K}(m^{0,K} - \nabla f(x^{0}))
        \\&&+ \alpha \sum_{t=1}^{r-1} \sum_{l=0}^{K-1} (1-\alpha)^{(r-t-1)K-l-1} (g^{t,l} - \nabla f^{t,l})
        \\&&+ \alpha \sum_{t=1}^{r-1} \sum_{l=0}^{K-1} (1-\alpha)^{(r-t-1)K-l-1}(\nabla f^{t,l} - \nabla f(x^t))
        \\&& + \sum_{t=0}^{r-1} (1-\alpha)^{K(r-t)}(\nabla f(x^{t}) - \nabla f(x^{t+1}))
        \\&=& (1-\alpha)^{rK}(m^{0,0} - \nabla f(x^{0}))
        \\&&+ \alpha \sum_{t=0}^{r-1} \sum_{l=0}^{K-1} (1-\alpha)^{(r-t)K-l-1} (g^{t,l} - \nabla f^{t,l})
        \\&&+ \alpha \sum_{t=0}^{r-1} \sum_{l=0}^{K-1} (1-\alpha)^{(r-t)K-l-1}(\nabla f^{t,l} - \nabla f(x^t))
        \\&& + \sum_{t=0}^{r-1} (1-\alpha)^{K(r-t)}(\nabla f(x^{t}) - \nabla f(x^{t+1}))
        .
    \end{eqnarray*}
    We will separately bound the expected norm of the terms. The first one equlas to $0$ due to proper initialization.
    The last two terms are bounded by the Assumption \ref{ass:smooth} and Lemmas \ref{lem:dist},\ref{lem:step}:
    \begin{eqnarray*}
        &&\mathbb{E}\big\|\alpha \sum_{t=0}^{r-1} \sum_{l=0}^{K-1} (1-\alpha)^{(r-t)K-l-1}(\nabla f^{t,l} - \nabla f(x^t))
        \\&& + \sum_{t=0}^{r-1} (1-\alpha)^{K(r-t)}(\nabla f(x^{t} - \nabla f(x^{t+1}))\big\|_*
        \\&\overset{(\ref{Norm})}{\leq} & \frac{\alpha}{M} \sum_{i=1}^M \sum_{t=0}^{r-1} \sum_{l=0}^{K-1} (1-\alpha)^{(r-t)K-l-1}\mathbb{E}\|\nabla f_i^{t,l} - \nabla f_i(x^t)\|_*
        \\&& + \sum_{t=0}^{r-1} (1-\alpha)^{K(r-t)}\|\nabla f_i(x^{t}) - \nabla f_i(x^{t+1})\|_*
        \\&\overset{(Ass.\ref{ass:smooth})}{\leq} & \frac{\alpha}{M} \sum_{i=1}^M \sum_{t=0}^{r-1} \sum_{l=0}^{K-1} (1-\alpha)^{(r-t)K-l-1}L_*\mathbb{E}\|x_i^{t,l} - x^t\|
        \\&& + \sum_{t=0}^{r-1} (1-\alpha)^{K(r-t)}L_*\mathbb{E}\|x^{t} - x^{t+1})\|
        \\&\overset{(Lem.\ref{lem:dist}m\ref{lem:step})}{\leq} & \frac{\alpha}{M} \sum_{i=1}^M \sum_{t=0}^{r-1} \sum_{l=0}^{K-1} (1-\alpha)^{(r-t)K-l-1}L_*\eta K (1+\varepsilon_q)
        \\&& + \sum_{t=0}^{r-1} (1-\alpha)^{K(r-t)}\cdot 6L_*\eta K \beta
        \\&{\leq} & 2 \alpha L_*\eta K\sum_{n=0}^{\infty} (1-\alpha)^{n}
        \\&& +6 L_*\eta K \beta \sum_{n=0}^{\infty} (1-\alpha)^{Kn}
        \\& = & 2 L_*\eta K +\frac{6 L_*\eta K \beta}{1-(1-\alpha)^K}
        .
    \end{eqnarray*} 
    
    The second term can be bounded via transition to the Frobenius norm.
    \begin{eqnarray*}
        &&\mathbb{E}\big\|\alpha \sum_{t=0}^{r-1} \sum_{l=0}^{K-1} (1-\alpha)^{(r-t)K-l-1} (g^{t,l} - \nabla f^{t,l})\big\|_* 
        \\&\leq&\rho \mathbb{E}\big\|\alpha \sum_{t=0}^{r-1} \sum_{l=0}^{K-1} (1-\alpha)^{(r-t)K-l-1} (g^{t,l} - \nabla f^{t,l})\big\|_F.
    \end{eqnarray*}
    Next, we apply Jensen inequality to the expectation:
    \begin{eqnarray*}
        &&\rho \mathbb{E}\big\|\alpha \sum_{t=0}^{r-1} \sum_{l=0}^{K-1} (1-\alpha)^{(r-t)K-l-1} (g^{t,l} - \nabla f^{t,l})\big\|_F
        \\&\leq & \rho \sqrt{\mathbb{E}\big\|\alpha \sum_{t=0}^{r-1} \sum_{l=0}^{K-1} (1-\alpha)^{(r-t)K-l-1} (g^{t,l} - \nabla f^{t,l})\big\|_F^2}
        .
    \end{eqnarray*}
    The Frobenius norm can be presented as a scalar product:
    \begin{eqnarray*}
        &&\mathbb{E}\big\|\alpha \sum_{t=0}^{r-1} \sum_{l=0}^{K-1} (1-\alpha)^{(r-t)K-l-1} (g^{t,l} - \nabla f^{t,l})\big\|_F^2
        \\&=& \alpha^2 \sum_{t,t'=0}^{r-1} \sum_{l,l'=0}^{K-1} (1-\alpha)^{(2r-t-t')K-l-l'-1} \mathbb{E}\bigg\langle g^{t,l} - \nabla f^{t,l},g^{t',l'} - \nabla f^{t',l'}\bigg\rangle
        \\&=& \frac{\alpha^2}{M^2}\sum_{i,i'=1}^M \sum_{t,t'=0}^{r-1} \sum_{l,l'=0}^{K-1} (1-\alpha)^{(2r-t-t')K-l-l'-1} \mathbb{E}\bigg\langle g_i^{t,l} - \nabla f_i^{t,l},g_{i'}^{t',l'} - \nabla f_{i'}^{t',l'}\bigg\rangle
        .
    \end{eqnarray*}
    We separate the sum in two cases: $S_1$ denotes all scalar products, where $i=i',t=t',l=l'$, while $S_2$ denotes the rest of the sum. Since the stochastic gradients are indenendent, all terms in $S_2$ equal to $0$:
    \begin{equation*}
        \mathbb{E}\bigg\langle g_i^{t,l} - \nabla f_i^{t,l},g_{i'}^{t',l'} - \nabla f_{i'}^{t',l'}\bigg\rangle = \bigg\langle \mathbb{E}[g_i^{t,l} - \nabla f_i^{t,l}],\mathbb{E}[g_{i'}^{t',l'} - \nabla f_{i'}^{t',l'}]\bigg\rangle = \langle 0,0\rangle = 0.
    \end{equation*}
    $S_2$ can be bounded via Assumption \ref{ass:smooth}.
    \begin{eqnarray*}
        S_2 &=& \frac{\alpha^2}{M^2}\sum_{i=1}^M \sum_{t=0}^{r-1} \sum_{l=0}^{K-1} (1-\alpha)^{2(r-t)K-2l-2} \mathbb{E}\| g_i^{t,l} - \nabla f_i^{t,l}\|_F^2
        \\&\leq & \frac{\alpha^2}{M^2}\sum_{i=1}^M \sum_{t=0}^{r-1} \sum_{l=0}^{K-1} (1-\alpha)^{2(r-t)K-2l-2} \sigma^2
        \\&\leq & \frac{\alpha^2\sigma^2}{M} \sum_{t=0}^{\infty} \sum_{l=0}^{K-1} (1-\alpha)^{2(r-t)K-2l-2}
        \\&= & \frac{\alpha\sigma^2}{M}.
    \end{eqnarray*}
    Thus, returning to the conjugate norm, we obtain the following.
    \begin{eqnarray*}
        &&\mathbb{E}\big\|\alpha \sum_{t=0}^{r-1} \sum_{l=0}^{K-1} (1-\alpha)^{(r-t)K-l-1} (g^{t,l} - \nabla f^{t,l})\big\|_* 
        \leq\rho\sqrt{\frac{\alpha\sigma^2}{M}} = \rho \sigma \sqrt{\frac{\alpha}{M}}.
    \end{eqnarray*}
    Combining the results yields the Lemma statement.
\end{proof}
\begin{corollary}\label{cor:mi_fi}
    The same bound holds for a single client's momentum without the $\frac{1}{\sqrt{M}}$ factor:
    \begin{eqnarray*}
        \mathbb{E}\|m^{r,K} - \nabla f(x^r)\|_* \leq {\rho \sigma}{\sqrt\alpha} + 2 L_*\eta K +\frac{6 L_*\eta K \beta}{1-(1-\alpha)^K}.
    \end{eqnarray*}
\end{corollary}
\begin{lemma}\label{lem:c_f}
    Suppose Assumptions \ref{ass:smooth}, \ref{ass:non-convex}, and \ref{ass:stoch} hold. Then durin the run of Algorithm \ref{alg:fedcalm} the following holds:
    \begin{eqnarray*}
    \mathbb{E}\|c^r - \nabla f(x^r)\|_* \leq 3\rho \sigma\beta \sqrt\alpha + 6 L_*\eta K \beta+\frac{18 L_*\eta K \beta^2}{\alpha}.
    \end{eqnarray*}
\end{lemma}
\begin{proof}
    The cornerstone of the theorem is a proper bound for $\|c^r -\nabla f(x^r)\|_*. $ We start applying Lemma \ref{lem:m_f}:
    \begin{eqnarray}\label{eq:c_f}
        \mathbb{E}\|c^r - \nabla f(x^r)\|_*& \overset{(\ref{Norm})}{\leq}& \mathbb{E}\|c^r - m^{r,0}\|_* + \mathbb{E}\|m^{r,0} - \nabla f(x^r)\|_*
        \\&\leq &\mathbb{E}\|c^r - m^{r,0}\|_* + \rho \sigma\sqrt{\frac\alpha M} + 2 L_*\eta K +\frac{6 L_*\eta K \beta}{1-(1-\alpha)^K}
        .
    \end{eqnarray}
    In order to estimate $\mathbb{E}\|c^r - m^{r,0}\|_*$ in (\ref{eq:c_f}), we separate the clients with (\ref{Jen}) and plug in the definition of $c_i^r$.
    \begin{eqnarray*}
        \mathbb{E}\|c_i^r - m_i^{r,0}\|_* &= &\mathbb{E} \left\|c_i^{r-1} + C\left(m_i^{r-1,K} - c_i^r\right) - m_i^{r-1,K}\right\|_*
        \\&\overset{(Def.\ref{eq:compres})}{\leq} & (1-\frac1\beta)\mathbb{E}\left\|c_i^{r-1} - m_i^{r-1,K}\right\|_*
        \\&\overset{(\ref{Norm})}{\leq} & (1-\frac1\beta)\mathbb{E}\left\|c_i^{r-1} - m_i^{r-1,0}\right\|_* + (1-\frac1\beta)\mathbb{E}\|m_i^{r-1,K} - m_i^{r-1,0}\|_*
        .
    \end{eqnarray*}
    The difference of the momentum through the round deserves special attention.
    \begin{eqnarray*}
        m_i^{r,K} - m_i^{r,0} & = & (1-\alpha)m_i^{r,K-1} + \alpha (g_i^{r,K-1} - \nabla f_i^{r,K-1}) +\alpha \nabla f_i^{r,K-1} - m_i^{r,0}
        \\& = & \dots
        \\& = & (1-\alpha)^K m_i^{r,0} + \alpha \sum_{l=0}^{K-1}(1-\alpha)^{K-l-1} (g_i^{r,l} - \nabla f_i^{r,l}) 
        \\&& + \alpha\sum_{l=0}^{K-1} (1-\alpha)^{K-l-1} \nabla f_i^{r,l} - m_i^{r,0}
        \\& = & -(1-(1-\alpha)^K)m_i^{r,0} + \alpha \sum_{l=0}^{K-1}(1-\alpha)^{K-l-1} (g_i^{r,l} - \nabla f_i^{r,l}) 
        \\&& + (1-(1-\alpha))\sum_{l=0}^{K-1} (1-\alpha)^{K-l-1} \nabla f_i^{r,l}
        \\& = & (1-(1-\alpha)^K)(\nabla f_i^{r,0} - m_i^{r,0}) + \alpha \sum_{l=0}^{K-1}(1-\alpha)^{K-l-1} (g_i^{r,l} - \nabla f_i^{r,l}) 
        \\&& + (\nabla f_i^{r,K-1} -\nabla f_i^{r,0}) + \sum_{l=0}^{K-2} (1-\alpha)^{K-l-1} (\nabla f_i^{r,l} - \nabla f_i^{r,l+1})
        .
    \end{eqnarray*}
    The first term is bounded via Corrolary \ref{cor:mi_fi} and Bernoulli inequality $(1-\alpha)^K\geq 1-\alpha K$, which leads to $1-(1-\alpha)^K \leq \min\{\alpha K;1\}$.
    \begin{eqnarray*}
        (1-(1-\alpha)^K)(\nabla f_i^{r,0} - m_i^{r,0}) &\leq& (1-(1-\alpha)^K)({\rho \sigma}{\sqrt\alpha} + 2 L_*\eta K +\frac{6 L_*\eta K \beta}{1-(1-\alpha)^K})
        \\&\leq& \rho \sigma K\alpha^{\nicefrac{3}{2}} + 2 L_*\eta K + {6 L_*\eta K \beta}.
    \end{eqnarray*}
    Similarly to Lemma \ref{lem:m_f}, we can bound the remaining terms.
        \begin{eqnarray*}
        &&\mathbb{E}\bigg\|\alpha \sum_{l=0}^{K-1}(1-\alpha)^{K-l-1} (g_i^{r,l} - \nabla f_i^{r,l}) 
        \\&& + \alpha\sum_{l=0}^{K-1} (1-\alpha)^{K-l-1} (\nabla f_i^{r,l} - \nabla f_i^{r,0}) + \nabla f_i^{r,K-1} -\nabla f_i^{r,0}\bigg\|_*
        \\&\overset{(Ass.\ref{ass:smooth})}{\leq} & \alpha\sqrt{\mathbb{E}\left\|\sum_{l=0}^{K-1}(1-\alpha)^{K-l-1} (g_i^{r,l} - \nabla f_i^{r,l})\right\|_F^2}
        \\&&+ \alpha\sum_{l=0}^{K-1} (1-\alpha)^{K-l-1} \|\nabla f_i^{r,l} - \nabla f_i^{r,0}\|_* + \| \nabla f_i^{r,K-1} -\nabla f_i^{r,0}\|_*
        \\&\overset{(Lem.\ref{lem:dist})}{\leq} & \alpha\rho\sigma\sqrt{K}
        + \alpha\sum_{l=0}^{\infty} (1-\alpha)^l L_* \eta(1+\varepsilon_q) + L_*\eta K (1+\varepsilon_q)
        \\&{\leq} & \alpha\rho\sigma\sqrt{K}
        + 4 L_* \eta K
        .
    \end{eqnarray*}
    Combining together the bounds, we have
    \begin{eqnarray*}
        \mathbb{E}\|m_i^{r,K} - m_i^{r,0}\|_* &\leq & \rho \sigma K\alpha^{\nicefrac{3}{2}} + 2 L_*\eta K + {6 L_*\eta K \beta} + \alpha\rho\sigma\sqrt{K}
        + 4 L_* \eta K
        \\&\leq& \rho \sigma \alpha\sqrt{K}(1 + \sqrt{\alpha K}) + 12 L_*\eta K \beta.
    \end{eqnarray*}
    The obtained result gives a recursive equation for the difference of the momentum and control variate:
    \begin{eqnarray*}
        \mathbb{E}\|c_i^r - m_i^{r,0}\|_* & \leq & (1-\frac1\beta)\mathbb{E}\left\|c_i^{r-1} - m_i^{r-1,0}\right\|_* + \rho \sigma \alpha\sqrt{K}(1 + \sqrt{\alpha K}) + 12 L_*\eta K \beta
        \\&\leq & \dots
        \\&\leq & (1-\frac1\beta)^r\mathbb{E}\left\|c_i^{0} - m_i^{0,0}\right\|_*
        \\&&+ 2\sum_{t=1}^{r}(1-\frac1\beta)^t(\rho \sigma \alpha\sqrt{K}(1 + \sqrt{\alpha K}) + 12 L_*\eta K \beta)
        \\&\leq & (1-\frac1\beta)^r\mathbb{E}\left\|c_i^{0} - m_i^{0,0}\right\|_* 
        \\&&+ (\rho \sigma \alpha\sqrt{K}(1 + \sqrt{\alpha K}) + 12 L_*\eta K \beta)\sum_{t=0}^{\infty}(1-\frac1\beta)^t
        \\&\leq & (1-\frac1\beta)^r\mathbb{E}\left\|c_i^{0} - m_i^{0,0}\right\|_* + \rho \sigma \alpha \beta \sqrt{K}(1 + \sqrt{\alpha K}) + 12 L_*\eta K \beta^2
        .
    \end{eqnarray*}
    Then, we note that the first term is zero and plug in the result into (\ref{eq:c_f}):
    \begin{eqnarray*}
        \mathbb{E}\|c^r - \nabla f(x^r)\|_*& \leq & (\rho \sigma \alpha \beta \sqrt{K}(1 + \sqrt{\alpha K}) + 12 L_*\eta K \beta^2)
        \\&&+ \rho \sigma\sqrt{\frac\alpha M} + 2 L_*\eta K +\frac{6 L_*\eta K \beta}{1-(1-\alpha)^K} 
        \\ & \leq & \rho \sigma\sqrt{\alpha}\left(\sqrt{\frac1 M} + \beta\sqrt{\alpha K} + \beta\alpha K \right) + 14 L_*\eta K \beta^2 +\frac{6 L_*\eta K \beta}{\alpha}
        .
    \end{eqnarray*}
\end{proof}
    
Finally, we return to the original theorem.
\begin{theorem}
    Suppose Assumptions \ref{ass:smooth}, \ref{ass:non-convex}, and \ref{ass:stoch} hold. Then Algorithm \ref{alg:fedcalm} to achieve $\varepsilon$-accuracy needs
    \begin{equation*}
     R = \mathcal{O}\left(\frac{\beta^2\Delta L_*}{\varepsilon^2(1-\varepsilon_q)^2} + \frac{\rho^4\Delta L_*\sigma^2}{\varepsilon^4(1-\varepsilon_q)^4M}\right) \text{~communication rounds,}
    \end{equation*}
    where $\displaystyle\varepsilon = \mathbb{E}\left[\frac1R\sum_{r=0}^{R-1}\|\nabla f(x^r)\|_*\right]$, and $\Delta=f(x^0)-{f(x^*)}$, and $\varepsilon_q$ is the error of Newton-Schulz iterations.
\end{theorem}
\begin{proof}
    We start similarly to Theorem \ref{th:conv_fedcalm-d}, with smoothness inequality.
    \begin{eqnarray*}
        f(\tx^{r+1}) &\overset{(Ass.\ref{ass:smooth})}{\leq}& f(\tx^r) + \langle \nabla f(\tx^r),\tx^{r+1}-\tx^r\rangle + \frac{L_*}{2}\|\tx^{r+1}-\tx^r\|^2
        \\&\overset{(\ref{Norm})}{\leq}& f(\tx^r) - \sum_{i=1}^M\langle \nabla f(\tx^r),y_i^{r+1}\rangle + \frac{L_*}{2M}\sum_{i=1}^M\|y_i^{r+1}\|^2
        .
    \end{eqnarray*}
    We pay special attention to the scalar product.
    \begin{eqnarray*}
        -\langle \nabla f(\tx^r),y_i^{r+1}\rangle & = & \langle \nabla f(\tx^r),\sum_{k=0}^{K-1} \eta \operatorname{NS}^q(m_i^{r,k} - c_i^r + c^r) \rangle
        \\& = & \eta\sum_{k=0}^{K-1}\langle \nabla f(\tx^r), \operatorname{NS}^q(m_i^{r,k} - c_i^r + c^r) \rangle
        \\& = & \eta\sum_{k=0}^{K-1}\langle \nabla f(\tx^r), \operatorname{lmo}(m_i^{r,k} - c_i^r + c^r) \rangle
        \\&&+\eta\sum_{k=0}^{K-1}\langle \nabla f(\tx^r), [\operatorname{NS}^q-\operatorname{lmo}](m_i^{r,k} - c_i^r + c^r) \rangle
        \\& \overset{(\ref{Conj})}{\leq} & \eta\sum_{k=0}^{K-1}\langle \nabla f(\tx^r), \operatorname{lmo}(m_i^{r,k} - c_i^r + c^r) \rangle
        \\&&+\eta\sum_{k=0}^{K-1}\| \nabla f(\tx^r)\|_*\cdot \big\|[\operatorname{NS}^q-\operatorname{lmo}](m_i^{r,k} - c_i^r + c^r) \big\|
        \\& \overset{(\ref{eq:nss_bound})}{\leq} & \eta\sum_{k=0}^{K-1}\langle \nabla f(\tx^r), \operatorname{lmo}(m_i^{r,k} - c_i^r + c^r) \rangle
        +\eta\sum_{k=0}^{K-1}\| \nabla f(\tx^r)\|_*\cdot \varepsilon_q
        % \\& = & -\eta\sum_{k=0}^{K-1}\langle m_i^{r,k} - c_i^r + c^r, \operatorname{lmo}(m_i^{r,k} - c_i^r + c^r) \rangle
        % \\&&+\eta\sum_{k=0}^{K-1}\langle -\nabla f(\tx^r) + m_i^{r,k} - c_i^r + c^r, \operatorname{NS}^q(m_i^{r,k} - c_i^r + c^r) \rangle
        % \\& \overset{\eqref{}}{=} & -\eta\sum_{k=0}^{K-1}\langle m_i^{r,k} - c_i^r + c^r, \operatorname{NS}^q(m_i^{r,k} - c_i^r + c^r) \rangle
        % \\&&+\eta\sum_{k=0}^{K-1}\langle -\nabla f(\tx^r) + m_i^{r,k} - c_i^r + c^r, \operatorname{NS}^q(m_i^{r,k} - c_i^r + c^r) \rangle
    \end{eqnarray*}
    With inexact Newton-Schulz out of the way, we can apply main LMO property \eqref{lem:lmo}.
    \begin{eqnarray*}
        -\langle \nabla f(\tx^r),y_i^{r+1}\rangle &{\leq} & \eta\sum_{k=0}^{K-1}\langle \nabla f(\tx^r), \operatorname{lmo}(m_i^{r,k} - c_i^r + c^r) \rangle
        +\eta\sum_{k=0}^{K-1}\| \nabla f(\tx^r)\|_*\cdot \varepsilon_q
        \\& = & \eta\sum_{k=0}^{K-1}\langle m_i^{r,k} - c_i^r + c^r, \operatorname{lmo}(m_i^{r,k} - c_i^r + c^r) \rangle
        \\&&-\eta\sum_{k=0}^{K-1}\langle -\nabla f(\tx^r) + m_i^{r,k} - c_i^r + c^r, \operatorname{lmo}(m_i^{r,k} - c_i^r + c^r) \rangle
        \\&&+\eta\sum_{k=0}^{K-1}\| \nabla f(\tx^r)\|_*\cdot \varepsilon_q
        \\& \overset{(Lem.\ref{lem:lmo})}{=} & -\eta\sum_{k=0}^{K-1}\| m_i^{r,k} - c_i^r + c^r\|_*
        \\&&-\eta\sum_{k=0}^{K-1}\langle -\nabla f(\tx^r) + m_i^{r,k} - c_i^r + c^r, \operatorname{lmo}(m_i^{r,k} - c_i^r + c^r) \rangle
        \\&&+\eta\sum_{k=0}^{K-1}\| \nabla f(\tx^r)\|_*\cdot \varepsilon_q
        \\& \overset{(\ref{Conj})}{\leq} & -\eta\sum_{k=0}^{K-1}\| m_i^{r,k} - c_i^r + c^r\|_*
        \\&&+\eta\sum_{k=0}^{K-1}\|-\nabla f(\tx^r) + m_i^{r,k} - c_i^r + c^r\|_* \cdot \| \operatorname{lmo}(m_i^{r,k} - c_i^r + c^r) \|
        \\&&+\eta\sum_{k=0}^{K-1}\| \nabla f(\tx^r)\|_*\cdot \varepsilon_q
        \\& {\leq} & -\eta\sum_{k=0}^{K-1}\| m_i^{r,k} - c_i^r + c^r\|_*
        \\&&+\eta\sum_{k=0}^{K-1}\|-\nabla f(\tx^r) + m_i^{r,k} - c_i^r + c^r\|_* \cdot 1
        \\&&+\eta\sum_{k=0}^{K-1}\| \nabla f(\tx^r)\|_*\cdot \varepsilon_q
        \\& = & -\eta\sum_{k=0}^{K-1}\| m_i^{r,k} - c_i^r + c^r\pm \nabla f(x^r)\|_*
        \\&&+\eta\sum_{k=0}^{K-1}\|-\nabla f(\tx^r) + m_i^{r,k} - c_i^r + c^r\pm \nabla f(x^r)\|_* \cdot 1
        \\&&+\eta\sum_{k=0}^{K-1}\| \nabla f(\tx^r)\pm \nabla f(x^r)\|_*\cdot \varepsilon_q
        \\& \overset{(\ref{Norm})}{\leq} & -\eta\sum_{k=0}^{K-1}\|\nabla f(x^r)\|_*(1-\varepsilon_q)
        \\&&+2\eta\sum_{k=0}^{K-1}\|-\nabla f(x^r) + m_i^{r,k} - c_i^r + c^r\|_*
        \\&&+\eta(1+ \varepsilon_q)\sum_{k=0}^{K-1}\| \nabla f(\tx^r)- \nabla f(x^r)\|_*
        \\& \overset{(Ass.\ref{ass:smooth})}{\leq} & -\eta\sum_{k=0}^{K-1}\|\nabla f(x^r)\|_*(1-\varepsilon_q)
        +2\eta\sum_{k=0}^{K-1}\|-\nabla f(x^r) + m_i^{r,k} - c_i^r + c^r\|_*
        \\ &&+\eta(1+ \varepsilon_q)\sum_{k=0}^{K-1}L_*\| \tx^r-x^r\|
        .
    \end{eqnarray*}
    The last term can be easily bounded via Lemma \ref{lem:e}:
    \begin{eqnarray*}
        \eta(1+ \varepsilon_q)\sum_{k=0}^{K-1}L_*\| \tx^r-x^r\| &= & \eta(1+ \varepsilon_q)KL_*\left\|\frac1M\sum_{i=1}^Me_i^r\right\|
        \\&\overset{(\ref{Norm})}{\leq}&\eta(1+ \varepsilon_q)KL_*\frac1M\sum_{i=1}^M\left\|e_i^r\right\|
        \\&\overset{(Lem.\ref{lem:e})}{\leq}&\eta(1+ \varepsilon_q)KL_*\frac1M\sum_{i=1}^M\eta K \beta (1+\varepsilon_q)
        \\&=&\eta^2K^2\beta L_*(1+ \varepsilon_q)^2.
    \end{eqnarray*}
    % For the second term, we apply Lemma \ref{lem:c_f} and Corollary \ref{cor:c_f}.
    % \begin{eqnarray*}
    %     2\eta\sum_{k=0}^{K-1}\|-\nabla f(x^r) + m_i^{r,k} - c_i^r + c^r\|_* &\overset{(\ref{Norm})}{\leq}& 2\eta\sum_{k=0}^{K-1}\|\nabla f(x^r)- c^r\|_* + 2\eta\sum_{k=0}^{K-1}\|m_i^{r,k} - c_i^r\|_*
    %     \\&\leq & 2\eta\sum_{k=0}^{K-1}\left(3\rho \sigma\beta \sqrt\alpha + 6 L_*\eta K \beta+\frac{18 L_*\eta K \beta^2}{\alpha}\right) 
    %     \\&&+ 2\eta\sum_{k=0}^{K-1}\left(4\rho \sigma \beta\sqrt\alpha + 4 L_*\eta K\beta +\frac{12 L_*\eta K \beta^2}{\alpha}\right)
    %     \\&= & 10\eta K\left(\rho \sigma\beta \sqrt\alpha + 2 L_*\eta K \beta+\frac{6 L_*\eta K \beta^2}{\alpha}\right).
    %\end{eqnarray*}
    Then, we can return all the way to the scalar product:
    \begin{eqnarray*}
        -\langle \nabla f(\tx^r),y_i^{r+1}\rangle &{\leq} & -\eta\sum_{k=0}^{K-1}\|\nabla f(x^r)\|_*(1-\varepsilon_q)
        \\&&+2\eta\sum_{k=0}^{K-1}\|-\nabla f(x^r) + m_i^{r,k} - c_i^r + c^r\|_*
        +4\eta^2K^2\beta L_*
        % \\&\overset{(\ref{Norm})}{\leq} & -\eta\sum_{k=0}^{K-1}\|\nabla f(x^r)\|_*(1-\varepsilon_q)
        % +2\eta\sum_{k=0}^{K-1}\|m_i^{r,k} - c_i^r\|_*
        % \\&&+2\eta K\|\nabla f(x^r) - m^{r,0}\|_*
        % +2\eta K\|m^{r,0} - c^r\|_*
        % +4\eta^2K^2\beta L_*
        .
    \end{eqnarray*}
    Then, we plug in the bound for the scalar product into the first smoothness inequality:
    \begin{eqnarray*}
        f(\tx^{r+1}) &{\leq}& f(\tx^r) - \frac1M\sum_{i=1}^M\langle \nabla f(\tx^r),y_i^{r+1}\rangle + \frac{L_*}{2M}\sum_{i=1}^M\|y_i^{r+1}\|^2
        \\&\leq & f(\tx^r) + \frac1M\sum_{i=1}^M \Bigg(-\eta\sum_{k=0}^{K-1}\|\nabla f(x^r)\|_*(1-\varepsilon_q) + 4L_*\eta^2K^2\beta
        \\&&+2\eta\sum_{k=0}^{K-1}\|-\nabla f(x^r) + m_i^{r,k} - c_i^r + c^r\|_*\Bigg) + \frac{L_*}{2M}\sum_{i=1}^M\|y_i^{r+1}\|^2
        \\&\leq & f(\tx^r) -\eta K\|\nabla f(x^r)\|_*(1-\varepsilon_q) + 4L_*\eta^2K^2\beta
        \\&&+\frac{2\eta}{M} \sum_{i=1}^M\sum_{k=0}^{K-1}\|-\nabla f(x^r) + m_i^{r,k} - c_i^r + c^r\|_* + \frac{L_*}{2M}\sum_{i=1}^M\|y_i^{r+1}\|^2
        .
    \end{eqnarray*}
    Further, Corollary \ref{cor:y} bounds the last term:
    \begin{eqnarray*}
        \frac{L_*}{2M}\sum_{i=1}^M\|y_i^{r+1}\|^2 \leq \frac{L_*}{2M}\sum_{i=1}^M (\eta K (1+\varepsilon_q))^2 \leq 2L_* \eta^2 K^2.
    \end{eqnarray*}
    We regroup the terms and take the average over the rounds $r=0\dots R-1$.
    \begin{eqnarray*}
        \eta K(1-\varepsilon_q)\frac1R\sum_{r=0}^{R-1} \|\nabla f(x^r)\|_* &\leq & \frac1R\sum_{r=0}^{R-1}(f(\tx^r) - f(\tx^{r+1}))
        \\&&+ 4L_*\eta^2K^2\beta + 2L_* \eta^2 K^2
        \\&&+\frac{2\eta}{MR} \sum_{r=0}^{R-1}\sum_{i=1}^M\sum_{k=0}^{K-1}\|-\nabla f(x^r) + m_i^{r,k} - c_i^r + c^r\|_*
        .
    \end{eqnarray*}    
    Next, we define $f(\tx^0) - f(\tx^R)\leq f(\tx^0)-f(x^*)=:\Delta$ and divide the equation by $\eta K(1-\varepsilon_q)$:
    \begin{eqnarray*}
        \frac1R\sum_{r=0}^{R-1} \|\nabla f(x^r)\|_* &\leq & \frac{\Delta}{\eta K R(1-\varepsilon_q)}+ \frac{6L_*\eta K\beta}{1-\varepsilon_q}
        \\&&+\frac{2}{MRK(1-\varepsilon_q)} \sum_{r=0}^{R-1}\sum_{i=1}^M\sum_{k=0}^{K-1}\|-\nabla f(x^r) + m_i^{r,k} - c_i^r + c^r\|_*
        .
    \end{eqnarray*}
    We take the expectation and utilize the Lemmas \ref{lem:c_f},\ref{lem:m_f} to bound the last term. We analyse the norm:
    \begin{eqnarray*}
        &&\frac{2}{MRK(1-\varepsilon_q)} \sum_{r=0}^{R-1}\sum_{i=1}^M\sum_{k=0}^{K-1}\mathbb{E}\|-\nabla f(x^r) + m_i^{r,k} - c_i^r + c^r\|_* 
        \\&\overset{(\ref{Norm})}{\leq}& \frac{2}{MRK(1-\varepsilon_q)} \sum_{r=0}^{R-1}\sum_{i=1}^M\sum_{k=0}^{K-1}(\mathbb{E}\|m_i^{r,k} - c_i^r\|_* + \mathbb{E}\|c^r - \nabla f(x^r)\|_*) 
        \\&\overset{(Lem.\ref{lem:c_f})}{\leq}& \frac{2}{1-\varepsilon_q}\left(\rho \sigma\sqrt{\alpha}\left(\sqrt{\frac1 M} + 2\beta\sqrt{\alpha K} + 2\beta\alpha K \right) + 28 L_*\eta K \beta^2 +\frac{6 L_*\eta K \beta}{\alpha}\right)
        .
    \end{eqnarray*}
   
    Now, we choose the parameters.    
    \begin{equation*}
        \eta = \min\left\{\frac{\sqrt{\Delta}}{K\sqrt{R L_* \beta}}, \sqrt{\frac{\Delta \alpha}{K^2\beta^2RL_*}}\right\}, \alpha = \sqrt{\frac{\Delta L_*\beta}{RM\rho^2\sigma^2}}
        .    
    \end{equation*}
    We note that since $\alpha$ decreases over $R$, in the bracket $\left(\sqrt{\frac1 M} + 2\beta\sqrt{\alpha K} + 2\beta\alpha K \right)$ the first term is leading. Thus, we obtain the bounds as follows.
    \begin{eqnarray*}
        \mathbb{E}\left[\frac1R\sum_{r=0}^{R-1} \|\nabla f(x^r)\|_*\right] &\leq & \frac{\Delta}{\eta K R(1-\varepsilon_q)}+ \frac{6L_*\eta K\beta}{1-\varepsilon_q}
        \\&&+\frac{2}{1-\varepsilon_q}\left(\rho \sigma\sqrt{\frac\alpha M} + 28 L_*\eta K \beta^2 +\frac{6 L_*\eta K \beta}{\alpha}\right)
        \\&\leq & \frac{\Delta}{\eta K R(1-\varepsilon_q)}+ \frac{62L_*\eta K\beta^2}{1-\varepsilon_q}
        \\&&+\frac{2\rho \sigma}{1-\varepsilon_q}\sqrt{\frac\alpha M} +\frac{6 L_*\eta K \beta}{\alpha(1-\varepsilon_q)}
        .
    \end{eqnarray*}
    To analyse the round complexity, we plug in the minima:
    \begin{eqnarray*}
        \mathbb{E}\left[\frac1R\sum_{r=0}^{R-1} \|\nabla f(x^r)\|_*\right] &\leq & \frac{\Delta}{\eta K R(1-\varepsilon_q)}+ \frac{62L_*\eta K\beta^2}{1-\varepsilon_q}
        \\&&+\frac{2\rho \sigma}{1-\varepsilon_q}\sqrt{\frac\alpha M} +\frac{6 L_*\eta K \beta}{\alpha(1-\varepsilon_q)}
        \\&\leq & \frac{\Delta}{K R(1-\varepsilon_q)}\left(\frac{K\sqrt{R L_* \beta^2}}{\sqrt{\Delta}} + \sqrt{\frac{K^2\beta^2RL_*}{\Delta \alpha}}\right)
        \\&&+ \frac{62L_* K\beta^2}{1-\varepsilon_q}\frac{\sqrt{\Delta}}{K\sqrt{R L_* \beta^2}}
        \\&&+\frac{2\rho \sigma}{1-\varepsilon_q}\sqrt{\frac\alpha M} +\frac{6 L_*K \beta}{\alpha(1-\varepsilon_q)}\cdot\sqrt{\frac{\Delta \alpha}{K^2\beta^2RL_*}}
        \\&=&\frac{63\beta}{(1-\varepsilon_q)}\sqrt{\frac{\Delta L_*}{R}}
        +\frac{2\rho \sigma}{1-\varepsilon_q}\sqrt{\frac\alpha M} +\frac{6 \sqrt{\Delta L_*\beta}}{\sqrt{R\alpha}(1-\varepsilon_q)}
        \\&=&\frac{63\beta}{(1-\varepsilon_q)}\sqrt{\frac{\Delta L_*}{R}}
        +\frac{8}{1-\varepsilon_q}\left(\frac{\rho^2\sigma^2\beta}{RM}\right)^{\nicefrac{1}{4}}
        .
    \end{eqnarray*}
    Rearranging the terms, we arrive at the Theorem statement.
    
    \begin{eqnarray*}
        R & = & \mathcal{O}\left( \frac{{\Delta L_* \beta^2}}{\varepsilon^2(1-\varepsilon_q)^2} + \frac{\Delta L_*\rho^2\sigma^2\beta}{M\varepsilon^4(1-\varepsilon_q)^4} \right)
        .
    \end{eqnarray*}
\end{proof} 
\end{appendixpart}

\end{document}